\documentclass[sigconf]{aamas}

\usepackage{balance} % for balancing columns on the final page

\usepackage{soul}
\usepackage[utf8]{inputenc}
\usepackage{booktabs}
\usepackage{bbm}
\usepackage{bm}
\usepackage{amsthm,amsmath}
\usepackage{mathrsfs}
\usepackage{xcolor}
\usepackage{appendix}
\usepackage{subfigure}
\usepackage{epsfig,verbatim,nicefrac}
\usepackage{IEEEtrantools}
\usepackage[T1]{fontenc}    % use 8-bit T1 fonts
\usepackage{amsfonts}       % blackboard math symbols
\usepackage{microtype}      % microtypography
\usepackage{enumitem}   %调整缩进
\allowdisplaybreaks[3]

\makeatletter

\newcommand{\Rmnum}[1]{\expandafter\@slowromancap\romannumeral #1@}
\makeatother

\newcommand{\sources}{\mathcal{S}} %set of sources
\newcommand{\options}{\mathcal{O}} %set of options
\newcommand{\reav}{\Upsilon} %one source realisation random variable
\newcommand{\rea}{\upsilon} %one source realisation sample
\newcommand{\reasetv}{\bm{\Upsilon}} %n sources realisation random variables
\newcommand{\reaset}{\bm{\upsilon}} %n sources realisation sample
\newcommand{\Reaset}{\mathcal{T}} %set of all realisations
\newcommand{\fb}{\bm{f}} %specific feedback
\newcommand{\Fb}{\mathcal{F}} %set of all feedback
\newcommand{\Fv}{\bm{F}} %a random variable with outcome feedback
\newcommand{\corre}{O}
\newcommand{\pv}{\bm{p}} %vector of trustworthiness values
\newcommand{\pdis}{P} %distribution of trustworthiness values
\newcommand{\pdisv}{\bm{P}}% joint distribution of vectors of trustworthiness values
\newcommand{\pest}{\hat{p}} %trust value
\newcommand{\pestv}{\hat{\bm{p}}} %vector of trust values
\newcommand{\pestdis}{\hat{P}} %distribution of trust values
\newcommand{\pestdisv}{\bm{\hat{P}}} %joint distribution of vectors of trust values

\newcommand{\trustworthiness}{trustworthiness}
\newcommand{\trust}{trust}
\newcommand{\correctness}{correctness}

\newcommand{\ds}{\mathcal{D}}

\newcommand{\dsvw}{\mathbb{D}_{W}}
\newcommand{\weight}{w}

\newcommand{\wmv}{\mathcal{D}_{W}}

\newcommand\prob{\mathbb{P}}
\newcommand\expect[1]{\mathbb{E}#1}
\newcommand{\dsc}{\omega}

\newcommand{\dscpep}[0]{\dsc(\pestv,\pv)}
\newcommand{\dscpp}[0]{\dsc(\pv,\pv)}
\newcommand{\dscqq}[0]{\dsc(\bm{q},\bm{q})}
\newcommand{\dscpepe}[0]{\dsc(\pestv,\pestv)}
\newcommand{\dscpedispdis}[0]{\dsc(\pestdisv,\pdisv)}
\newcommand{\dscpdispdis}[0]{\dsc(\pdisv,\pdisv)}
\newcommand{\dscpepdis}[0]{\dsc(\pestv,\pdisv)}
\newcommand{\dscpedisp}[0]{\dsc(\pestdisv,\pv)}

\newtheorem{theorem}{Theorem}
\newtheorem{definition}{Definition}

\newtheorem{lemma}{Lemma}
\newtheorem{corollary}{Corollary}

\newtheorem{myexp}{Example}

\setcopyright{ifaamas}
\acmConference[AAMAS '24]{Proc.\@ of the 23rd International Conference
on Autonomous Agents and Multiagent Systems (AAMAS 2024)}{May 6 -- 10, 2024}
{Auckland, New Zealand}{N.~Alechina, V.~Dignum, M.~Dastani, J.S.~Sichman (eds.)}
\copyrightyear{2024}
\acmYear{2024}
\acmDOI{}
\acmPrice{}
\acmISBN{}

\acmSubmissionID{397}

\title[AAMAS-2024 Formatting Instructions]{Stability of Weighted Majority Voting under Estimated Weights}

\author{Shaojie Bai}
\affiliation{
  \institution{Zhejiang University}
  \city{Hangzhou}
  \country{China}}
\email{white.shaojie@gmail.com}

\author{Dongxia Wang*}\thanks{* Corresponding author}
\affiliation{
  \institution{Zhejiang University \&}
  \institution{ZJU-Hangzhou Global Scientific and Technological Innovation Center}
  \city{Hangzhou}
  \country{China}}
\email{dxwang@zju.edu.cn}

\author{Tim Muller}
\affiliation{
  \institution{University of Nottingham}
  \city{Nottingham}
  \country{United Kingdom}}
\email{tim.muller@nottingham.ac.uk}

\author{Peng Cheng*}
\affiliation{
  \institution{Zhejiang University}
  \city{Hangzhou}
  \country{China}}
\email{saodiseng@gmail.com}

\author{Jiming Chen}
\affiliation{
  \institution{Zhejiang University}
  \city{Hangzhou}
  \country{China}}
\email{cjm@zju.edu.cn}

\begin{abstract}
\emph{Weighted Majority Voting} (WMV) is a well-known decision making rule. The weights of sources are determined by the probabilities that sources provide accurate information (\emph{trustworthiness}).
However, in reality, the trustworthiness is usually not a known quantity to the decision maker -- they have to rely on an estimate called \emph{trust}.
An algorithm that computes trust is called \emph{unbiased} when it has the property that it does not systematically overestimate or underestimate the trustworthiness. %move footnote
To formally analyze the uncertainty to the decision process brought by such unbiased trust values, we introduce and analyze two important properties of WMV: \emph{Stability of Correctness} and \emph{Stability of Optimality}.
Stability of Correctness measures the difference between the decision accuracy that the decision maker believes he can achieve and the accuracy he actually achieves.
We prove Stability of Correctness absolutely holds for WMV -- the difference is $0$.
Stability of Optimality measures the difference between the actual accuracy of decisions made using trust values, and those made using trustworthiness values.
We find a relatively tight upper bound on the Stability of Optimality, meaning that, although using (unbiased) trust values is suboptimal compared to using the true trustworthiness values, the difference is small.
Meanwhile, a counter-intuitive observation is that while distributions of trustworthiness influence the Stability of Optimality, the number of sources barely influences it.
We also provide an overview of how sensitive decision accuracy is to the changes in trust and trustworthiness.
\end{abstract}


\newcommand{\BibTeX}{\rm B\kern-.05em{\sc i\kern-.025em b}\kern-.08em\TeX}

\makeatletter
\gdef\@copyrightpermission{
	\begin{minipage}{0.3\columnwidth}
		\href{https://creativecommons.org/licenses/by/4.0/}{\includegraphics[width=0.90\textwidth]{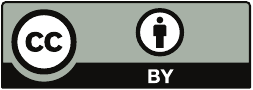}}
	\end{minipage}\hfill
	\begin{minipage}{0.7\columnwidth}
		\href{https://creativecommons.org/licenses/by/4.0/}{This work is licensed under a Creative Commons Attribution International 4.0 License.}
	\end{minipage}
	\vspace{5pt}
}
\makeatother

\begin{document}

%%% The following commands remove the headers in your paper. For final 
%%% papers, these will be inserted during the pagination process.

\pagestyle{fancy}
\fancyhead{}

%%% The next command prints the information defined in the preamble.

\maketitle 

%%%%%%%%%%%%%%%%%%%%%%%%%%%%%%%%%%%%%%%%%%%%%%%%%%%%%%%%%%%%%%%%%%%%%%%%

\section{Introduction}
Crowd wisdom has been playing a fundamental role in helping make better decisions in many scenarios, e.g., hiring workers for labeling tasks in crowdsourcing \cite{luo2023incentivizing}, aggregating classifiers for prediction in ensemble learning \cite{kotary2023differentiable}, asking for opinions of reliability in online rating systems \cite{carbo2023promoting,ge2023trustworthiness}, etc.
Decisions are derived based on aggregating the information or feedback from a collection of sources, the quality of which can be variable. It can be inaccurate due to lack of expertise, mistakes or malice, e.g., low-quality labeling for machine learning, fake ratings introduced by sellers to promote their reputation, etc. 

Among the aggregation mechanisms, Weighted Majority Voting has long been a popular one.
Basically, each source supports an option and is assigned a weight.
WMV chooses the feedback option that is supported by sources with the maximal total weight. 
WMV has been seeing its use in a variety of domains ranging from voting \cite{nitzan:1982optimal}, crowdsourcing \cite{dawid:1979DGModel}, classification \cite{LITTLESTONE:1994WMAlgorithm} to trust systems \cite{yu2004:WMV_in_trust} and even distributed systems \cite{tong:1991DSystem_vote}.
In different contexts, the weight of a source can mean differently.
For instance, in determining a collective choice that is widely acceptable to individuals with diverse preference~\cite{sen1977social},
% \footnote{For example, voting for democracy in social choice~\cite{sen1977social}, or large-scale group decision making problems (LSGDM)~\cite{chen2022public}.}
WMV is used for preference aggregation and the weight means the importance of an individual.
We are more interested in the scenarios where there is a notion of \emph{correctness} (or accuracy) of decisions, and the weight of a source depends on how trustworthy it is in providing the feedback that corresponds to the correct decision, denoted as \emph{\trustworthiness}, which is usually modelled as a probability value.
The examples include the aggregation of the crowd-sourced labels~\cite{li:2014errorbound_WMV}, the crowd-sensed navigation data~\cite{james2020sybil}, or the outputs of the multiple classifiers~\cite{manino:2019NaiveByes_Multiclass}, etc.
% The examples include whether the aggregation of the crowd-sourced labels~\cite{li:2014errorbound_WMV}, the crowd-sensed navigation data~\cite{james2020sybil}, or the outputs of the multiple classifers are accurate~\cite{manino:2019NaiveByes_Multiclass}, etc.

% The weight is determined by how trustworthy a source is in suggesting the correct decisions (or providing feedback that corresponds to the correct decision, denoted as \emph{\trustworthiness}), which can be modelled as a probability value.
% \sj{Note that this scenario is entirely different from the traditional social choice problem of determining the desirable collective choice given diverse individual preferences \cite{sen1977social}, where a "correct" decision may not exist.}

While WMV is proven to be optimal when source \emph{trustworthiness} is given \cite{nitzan:1982optimal}\footnote{To provide feedback independently is also required for the optimality of WMV~\cite{muller:2020MPR,nitzan:1982optimal}.}, in practice, decision makers have to resort to an estimation or a belief (denoted as \emph{trust}) that may not equal to the actual values of \trustworthiness{}\footnote{Note that we use ``trust'' and ``trustworthiness'' to differentiate between what a decision maker trusts or estimates as the probability of a source's suggesting correctly (regardless of whether he believes in the value or is aware that its just an estimate), and the actual probability value (Refer to their definition difference~\cite{walter:2008cambridge_dictionary}).}. 
%For example, if ``trustworthiness'' of a source denotes the probability of providing the correct feedback, then ``trust'' denotes the estimated value of that probability.}.
Deviation in estimation may decrease decision accuracy.
There exists plenty of effort to improve the estimation of source trustworthiness by learning from historical data (e.g., direct observations or indirect evidences), with a principle that more data increases the confidence in the estimation~\cite{2013BEREND:estimate_error_0,wu2016:estimate_error_0}.
% \cite{2013BEREND:estimate_error_0}
Several approaches even treat the belief about source trustworthiness as its actual values~\cite{mazzetto:2021semi_supervised_MV,maystre:2021_noiselabel_crowdsourcing}.
However, no algorithm always produces perfect trust values.
It is worth studying how the quality of the trust values impacts the decision quality of WMV.
Is WMV able to maintain a tolerant level of decision incorrectness with the inaccuracy in the estimation bounded, meaning having certain levels of stability w.r.t the inaccurate estimation?

In this paper, we propose a formal analysis of the stability properties of WMV.
Firstly, we study how sensitive the decision accuracy is to the changes in source \trust{} and \trustworthiness{}, with both the arguments taking fixed values. 
We find that unsurprisingly, decision accuracy decreases with the increasing deviation from \trust{} to \trustworthiness{}, and a sufficiently small deviation barely influences the accuracy. 
Besides, compared with overestimating, underestimating \trustworthiness{} is usually less harmful to the decision accuracy.
Secondly, we study the influence of \trust{} and \trustworthiness{} in a statistical way. 
Considering that a decision maker may sometimes overestimate source \trustworthiness{} while sometimes underestimate it, the expectation remains correct -- unbiased estimation\footnote{Generally, the estimation error always exists, but it is relatively small and can be zero on average with sufficient data \cite{2013BEREND:estimate_error_0,freedman1963:consistency_bayes}}.
We define two types of stability based on such unbiased estimation: \emph{Stability of Correctness} and \emph{Stability of Optimality}. 
\emph{Stability of Correctness} reasons whether the decision accuracy a decision maker believes he achieves (i.e., the accuracy he computes with \trust) equals what he actually achieves (i.e., the accuracy computed with \trustworthiness{}). 
We prove that whatever distribution source \trustworthiness{} follows, as long as the estimation is unbiased, a decision maker gets the accuracy as if the \trustworthiness{} is known -- absolute stability. 
This means that the shape and variance of \trustworthiness{} are irrelevant to the Stability of Correctness.
% \emph{Stability of Optimality} reasons how much room for the improvement of decision accuracy there is, when there is better estimation. 
% We prove an upper bound of how much decision accuracy can be improved if better estimation is provided, and provide a quantitative reasoning of the relation between the estimation error of source \trustworthiness{} and the change in decision accuracy.

\emph{Stability of Optimality} reasons whether the decisions made based on unbiased trust values are as good as those made based on trustworthiness. 
%In practice, trustworthiness is usually unknown to a decision maker. 
Considering trustworthiness is usually unknown, Stability of Optimality measures the \emph{gap} between the practical situation where the decision maker decides with trust, and where (magically) he has access to the actual trustworthiness.
% We prove that unlike \emph{Stability of Correctness}, \emph{Stability of Optimality} does not hold for WMV, and also how much the decision accuracy can be improved with better estimation is bounded, indicating that it cannot be improved indefinitely.
% Moreover, unlike \emph{Stability of Correctness}, the distribution of \trustworthiness{} influences the upper bound of accuracy improvement.
% For example, \trustworthiness{} distributions with different ranges result in different upper bounds.  
We prove that \emph{Stability of Optimality} does not hold for WMV, but the degradation in decision accuracy caused by the incorrect but (averagely) unbiased \trust{} is relatively tightly bounded.
That is, decision accuracy with unbiased \trust{} will not be too far off the theoretically determined value.
Moreover, unlike \emph{Stability of Correctness}, the distribution of \trustworthiness{} influences the upper bound of that accuracy gap, and also determines how well the accuracy can be in the ideal situation, namely where \trustworthiness{} is given.
% For example, \trustworthiness{} distributions with different ranges result in different upper bounds.  
Last but not least, while it may usually be perceived that more sources improve accuracy, we observe counterintuitively that source number influences little on the accuracy gap.
% \dw{accuracy improvement or upper bound?}

The rest of this paper is organized as follows. In Section 2 the related work is presented. In Section 3 we introduce a formal framework to study WMV decision rule. In Section 4 we present how \trust{} and \trustworthiness{} influence the decision accuracy of WMV. In Section 5 we analyze the two types of stability. 
The numerical analysis is also performed where needed to demonstrate theories.

\section{Related Work}
The Weighted Majority Voting rule has been studied in several domains, e.g., decision theory, voting theory, management science, and receiving various applications. 
We focus on the scenarios where the weight of a source or ``voter'' is determined by how trustworthy it is in suggesting the correct decision.
% The Weighted Majority Voting rule is proposed by  \cite{nitzan:1982optimal}, where each voter is assigned a weight determined by the probability that he suggests the right decision option (i.e., known as source trustworthiness, competence, etc), and a decision is made by summing the weights. 
Some approaches utilizing WMV assume source trustworthiness is given~\cite{nitzan:1982optimal,berend:2015MPR-NIPS}, although in practice it is usually unknown.
Plenty of work focuses on modeling and learning source trustworthiness from observation and interaction history \cite{Zeynalvand21,wu2023crowdsourcing,ge2023trustworthiness}.
Some researchers model trust as a probability value. 
To get trust, they either rely on frequency estimation by counting the times of making the right decisions~\cite{2013BEREND:estimate_error_0}, or solving an optimization problem based on their models by minimizing the decision error rate \cite{rekatsinas2017slimfast} or maximizing the likelihood
\cite{dong:2015knowledge_WebSources,manino:2019Bayesian_crowdsourcing,meir2023frustratingly}.
Moreover, model-checking-based methods are also applied in quantifying the probability of trust on individual agents, representing the agent’s own beliefs \cite{drawel2020specification,bentahar2022quantitative,drawel2022formal,telang2023maintenance}.
Besides, trust also can be modeled as a random variable.
Bayesian models have also been proposed and applied to this problem by~\cite{raykar:2010learning_p_by_bayesian,sardana2018bayesian,guo2023multi}, combining the prior knowledge and the observations to infer the trustworthiness. 
Expectation Maximization-based methods are also proposed to estimate source trustworthiness and the correct decision at the same time, via iterative updating~\cite{dawid:1979DGModel,zhang:2016_Spectral_Methods_Meet_EM}.
% Other relevant methods include, e.g., triangular Estimation \cite{bonald:2017minimax_learn_p} through the correlation of the feedback of the source and those of the most informative pair.

Such learned trust is sometimes treated as an estimation of the source \trustworthiness{} with the deviation considered \cite{gao:2016NaiveBayes_Errorrate,wu:2021chebyshev_PACBounds_of_WMV}, while sometimes treated equivalently as \trustworthiness{}, namely as the probability of a source suggesting correctly and is further used to evaluate decision accuracy
~\cite{LITTLESTONE:1994WMAlgorithm,guan:2018_computer-aided_diagnosis,martin2023strong}.
% ~\cite{LITTLESTONE:1994WMAlgorithm,kolter:2007dynamic_WMV,guan:2018_computer-aided_diagnosis,martin2023strong}.
However, \trust{} essentially represents the belief of a decision maker about the source quality, which may deviate from the actual probability. 
And he may not gain the claimed decision accuracy based on \trust{}.

Besides the efforts in modeling and learning trustworthiness, there exists work that theoretically analyzes how trustworthiness and trust would influence decision accuracy, which is most relevant to ours. 
Given trustworthiness, the decision accuracy of WMV is analyzed in \cite{berend:2015MPR-NIPS} without considering the learning process of trustworthiness. 
On the other hand, some other work takes the learning process into consideration.
To measure the estimation quality, the decision accuracy bounds for learning algorithms have been proposed through PAC techniques in \cite{lacasse:2006pac_for_MV,germain:2015PAC_Boundsof_WMV,wu2021chebyshev}. 
% Some researchers then provide the exponential bounds of decision accuracy, when \trust{} is derived from finite samples~\cite{li:2014errorbound_WMV}.
% A preciser characterization of the relationship between the sample size and the decision accuracy is proposed by \cite{gao:2016NaiveBayes_Errorrate}.
Considering \trust{} is derived from finite samples, some researchers then study precise characterizations of the relationship between the decision accuracy and the sample size in \cite{gao:2016NaiveBayes_Errorrate}.
More recently, tighter bounds for decision accuracy under arbitrary estimation are provided, ignoring particular assumptions for trustworthiness ~\cite{manino:2019NaiveByes_Multiclass}.
Unfortunately, none of them have analyzed the relationship between the estimate error and the decision accuracy of WMV in a quantitative way.

\section{Preliminaries}
%\tm{I liked this paragraph. May put it back in, if we have space!}
% In many decision making situations, a decision maker obtains feedback (or vote, advice) from others, e,g, requesting votes for a selection, hiring workers for labeling or survey tasks in crowdsourcing, asking for opinions of reliability in online rating systems, etc. 
% While such feedback is supposed to inform our knowledge of the situation and help make better decisions, it is sometimes inaccurate due to lack of expertise, mistakes or malice, e.g., low-quality labeling for machine learning, misleading ratings introduced by sellers to promote their reputation.

In this section, we outline a formal framework to support our study of the stability of Weighted Majority Voting decision rule. Note that the capital letters represent random variables, and the lower cases represent non-random variables. The bold letters represent a vector of multiple variables, and the non-bold letters represent single variables.

% \dw{Reorganize: Clarify patterns of variables definitions, for each concept, define Set, random varaible, an outcome together, not in different places. And then, move to another concept.}

% Consider a decision making scenario, a decision maker is faced with two possible decisions $\options = \{+1, -1\}$ 
% % $\options = \{o_1, o_2\}$ 
% and only one of them is correct.
% The decision maker receives feedback from a set of sources: $\sources = \{s_1,..,s_n\}$.
% The feedback from the sources may or may not reflect the correct decision.
% Random variable $F_i$ represents the feedback or the decision option suggested by source $s_i$, with its outcome $f_i \in \options$.
% Feedback of all the sources is represented by random variable $\Fv$, with $\fb: \fb {\in} \Fb, \fb {=} (f_1, \dots, f_n)$ denoting an outcome.
% The random variable $C$ determines which of the two options is actually correct, e.g., $C {=} +1$ if $+1$ is the correct decision.
% A decision mechanism is a function: $\ds : \Fb \to \options$.
% The quantity that the decision maker wants to maximize is the probability of making the correct decision: $\prob(\ds(\Fv) = C)$.

Consider a decision-making scenario, a decision maker is faced with multiple possible decisions $\options = \{o_1, \dots, o_K\}$ and only one of them is correct.
The random variable $\corre$ determines which of the options is actually correct, e.g., $\corre {=} o_1$ if $o_1$ is the correct decision.
The decision maker receives feedback from a set of sources, $\sources = \{s_1,..,s_n\}$.
The random variable $F_i$, with $f_i$ denoting an outcome, represents the feedback of source $s_i$.
The feedback may or may not correspond to the correct decision.
For WMV, we assume\footnote{For model general decision-making scenarios, the options of feedback and that of decisions may not necessarily equal and may take a many-to-one mapping.} a one-to-one correspondence between the feedback that suggests the correct decision and the correct decision itself, and denote $F_i = \corre$ iff $f_i$ suggests correctly, $F_i \in \options$. 
Feedback of all the sources is represented by random variable $\Fv : \Fv = (F_1, \dots, F_n)$, with $\fb{:}\fb {=} (f_1, \dots, f_n)$ denoting an outcome, and its sample space is defined as $\Fb{:}\fb{\in}\Fb$.
A decision mechanism is a function: $\ds : \Fb {\to} \options$.
The quantity that the decision maker wants to maximize is the probability of making the correct decisions (which we shorthand as \emph{decision accuracy} or \emph{decision correctness} throughout the paper): $\prob(\ds(\Fv) = \corre)$.

%To facilitate our study, we reason the ``state'' of a source regarding whether it suggests the correct decision, instead of the specific feedback it provides.  
we define $\reav_i$ as a $\{-1,1\}$-indicator random variable of whether source $s_i$ suggests the correct decision and $\rea_i $ as one of its outcome: $\reav_i {=} 1$ if $F_i {=} \corre$ and $\reav_i {=} -1$ if $F_i {\neq} \corre$.
%Thus, $\prob(\reav_i {=} 1) {=} p_i$.
For the indicator variables of all the sources i.e., $\reasetv : \reasetv = (\reav_1, \dots, \reav_n)$, one of its samples is an \emph{indicator vector} i.e., $\reaset {=} (\rea_1,\dots,\rea_n), \reaset \in \Reaset$.
%The set of the indicator variables of all the sources is named as the \emph{realisation} variable (Refer to~\cite{muller:2020MPR}), denoted as . 
$\Reaset$ denote the sample space of $\reasetv$. 
%$\reaset {=} (\rea_1,\dots,\rea_n)$ is an indicator vector, $\reaset \in \Reaset$.
Let $-\reaset {=} (-\rea_1,\dots,-\rea_n)$ denote the \emph{opposite} indicator vector of $\reaset$ where source indicators are flipped.
The set of all the possible feedback under $\reaset$ is denoted as $\Fb_{\reaset}$.
%For example, given three sources and the correct decision is $-1$, if $\reasetv = \{1,1,-1\}$, then only the feedback $\Fv = (-1,-1,1)$ is in $\Fb_\reaset$.

We use the following running example in this section to demonstrate the relevant concepts.
\begin{myexp}\label{exp1}
There are three sources $\sources {=} \{s_1,s_2,s_3\}$. If $\options = \{A,B\}$, $\corre {=} B$ and the indicator vector $\reaset = (1,1,-1)$, then $\fb = (B,B,A)$ and $\Fb_{\reaset} = \{(B,B,A),(A,A,B)\}$.
\end{myexp}

When decision ``correctness'' is a concern, Weighted Majority Voting usually considers how probable each source suggests the correct decision.
For source $i$, let $\prob(F_i {=} \corre) = p_i$ and $\pv = (p_1,\dots,p_{n})$.
Hence $\prob(\reav_i = 1) = p_i$.
We refer to $p_i$ as the \emph{trustworthiness} of source $s_i$.
In practice, the trustworthiness of a source is usually unknown to a decision maker.
And an estimation is used, denoted as $\pest_i$, with $\pestv = (\pest_1, \dots \pest_n)$.
We call the value $\pest_i$ \emph{trust}, which represents the subjective estimation or belief of the decision maker regarding how probable a source suggests correctly.
There exist multiple ways to compute $\pest_i$, e.g., counting the frequency of making correct decisions, or Bayesian learning methods based on prior interaction data.
In the literature, the trustworthiness of a source can have different meanings, e.g. honesty of an agent in a rating system~\cite{muller:2020MPR},
% \dw{Not clear what kind of security system},
competency of a voter~\cite{Condorcet:2014essai}, reliability of a worker in crowdsourcing~\cite{dawid:1979DGModel}, correctness of a sensor in crowdsensing~\cite{moslem2012crowdsensing}, etc.
Whatever the meanings, $p_i$ represents an intrinsic quality or the fact that how probable the source reports correctly, while $\pest_i$ represents how the decision maker thinks of or estimates that probability~\cite{walter:2008cambridge_dictionary}. 
We assume sources independently provide feedback, hence $\prob(\reasetv = \reaset) = \prod_{i:\rea_i {=} 1} p_i \cdot \prod_{i:\rea_i {=} -1} (1{-}p_i)$.
% To formally analyze the probability of making correct decisions (which we shorthand as \emph{decision accuracy} or \emph{correctness} throughout the paper), 

In Example~\ref{exp1}, suppose $\pv = (0.6,0.6,0.7)$, the estimation of $\pv$ by a decision maker may be inaccurate: $\pestv = (0.6,0.7,0.8)$. 

% \sj{One given feedback $\fb$ can correspond to two opposite realisation $\reaset, -\reaset$, but only one of them will be chosen for the final decision, which characterizes that decision rule.}
% We can define $\dsv$ to be the set of realisations, such that if the feedback corresponding to a realisation leads to the correct decision, then the realisation is in $\dsv$: $\dsv = \{\reaset \in \Reaset | \ds(\Fv) = C\}$.
% It is typically useful to think about this set of realisation characterising a decision mechanism.
% The decision accuracy is.

% The probability a mechanism decides correctly is $\prob(\ds(\Fv) {=} C) = \sum_{\reaset \in \dsv} \prob(\reaset)$, called decision accuracy or correctness interchangeably.
%The notion of realisation will ease our reasoning of decision accuracy of a mechanism, compared with reasoning with feedback (the distribution of which is unknown), especially for $WMV$ decision mechanism.

% \dw{f needs to take binary value here, change it later.}
%Weighted Majority Voting is a kind of decision mechanisms, where when the ``correctness'' is an issue, the decision usually depends on the trustworthiness values:
Below, we introduce the Weighted Majority Voting (WMV) decision scheme.
It can be treated as an extension of the more commonly known \emph{Majority Voting} decision scheme.
The difference is that Majority Voting treats sources without distinguishing, while WMV assigns sources different weights. 
The weight of a source is usually determined by how trustworthy its feedback is.
Formally:
\begin{definition}[Weighted Majority Voting $\wmv$]
Given a set of $n$ sources $\sources$, their trustworthiness $\pv$ and independent feedback $\fb$, $\wmv$ makes decisions via the function \cite{muller:2020MPR,nitzan:1982optimal}:

% \begin{equation} \label{eq:WeightMajorityRule}
% \wmv(\fb)= \text{sign} \left(\sum_{i=1}^{n}{\weight_i \cdot f_i}\right)
% \end{equation}
% \\
% where $p_i\geq 0.5$, $f_i \in \{+1,-1\}$, and $\weight_i = \log (\nicefrac{p_i}{1 - p_i})$.
\begin{equation} \label{eq:WeightMajorityRule}
\wmv(\fb)= \text{argmax}_{o\in\options} \left(\sum_{i:f_i=o}{\weight_i}\right)
\end{equation}
\\
where $f_i \in \options$, $\weight_i = \log (\nicefrac{p_i}{1 - p_i})$ with $p_i\geq 0.5$.
\end{definition}

To give an instance, consider Example~\ref{exp1}, suppose $\pv = (0.6,0.6,0.9)$, $\options = (A,B)$, $\corre = B$ and $\reaset = (1,1,-1)$, then $\fb = (B,B,A)$.
$w_1 \approx 0.18, w_2 \approx 0.18, w_3 \approx 0.60$.
since $\weight_1 + \weight_2 < \weight_3$,  $\wmv(\fb) = A$.
%When $w_i$ equal for all the sources, Equation~\ref{eq:WeightMajorityRule} becomes majority rule.

Here $p_i \geq 0.5$ and the log weight function are well-known for classical WMV in the literature
\cite{nitzan:1982optimal,grofman1983thirteen}, where trust and trustworthiness are
not distinguished.
The assumption $p_i \geq 0.5$ means that sources with $p_i < 0.5$ are ignored.
For a source with $p_i < 0.5$, a decision maker may assign negative weight to its feedback.
Or he can just simply reverse the vote of the source (e.g., replacing the reported option A with C). 
But if either the operation is realized by the malicious sources, they can push the decision to a wrong one by reporting correctly, purposely reducing the chance of correct option being selected. 
Therefore, it is in the interest of the decision maker to ignore such sources.

It has been shown in the literature that the decision accuracy of WMV is determined by the indicator vectors where it always decides correctly (an example is where all sources report correctly).
%There are realisations where decisions are always correct (e.g., where $\forall s_i, \reav_i = 1$), the set of which varies for different decision mechanisms.
For such indicator vectors, whether a decision is correct is not influenced by the feedback of the sources that suggest incorrectly.
To give an opposite example, consider Example~\ref{exp1}.
Suppose $\options = \{A,B,C\}$, $\pv = (0.70,0.65,0.65)$ and $\reaset = (1,-1,-1)$ (only $s_1$ reports correctly). 
We get $(w_1,w_2,w_3) \approx (0.37,0.27,0.27)$.
Both the feedback $\fb = (A,B,C)$ and $\fb' = (A,C,C)$ are possible under $\reaset$ (both belong to $\Fb_{\reaset}$. 
However, WMV decides correctly by choosing $A$ under $\fb$ and decides incorrectly by choosing $C$ under $\fb'$. 
Whether WMV decides correctly is influenced by what incorrect feedback is.
Given the same $\pv$, for $\reaset' = (1,1,-1)$, it can be seen that whatever $s_3$ reports, WMV can always decides correctly by trusting $s_1,s_2$.

% Define $\dsv : \dsv = \{\reaset | \ds(\fb) = \corre, \fb {\in} \Fb_{\reaset} \}$.
% Then the decision accuracy of WMV can be computed as $\prob(\ds(\Fv) {=} C) = \sum_{\reaset \in \dsv} \prob(\reaset)$.

%For pair of indicator vectors $\reaset, -\reaset$, either $\reaset$ or $-\reaset$ belongs to $\dsv$ (Refer to~\cite{muller:2020MPR}).

%without the need of investigating the specific feedback.

Let $\dsvw (\pv)$ denote the set of all the indicator vectors where WMV always decides correctly when using $\pv$, namely $\dsvw(\pv) = \{\reaset | \wmv(\fb) = \corre, \fb {\in} \Fb_{\reaset} \}$.
It has been proven that $\dsvw(\pv) = \{\reaset | \prob(\reaset) \geq \prob(-\reaset)\}$ and the decision accuracy of $\wmv$ is (Refer to ~\cite{nitzan:1982optimal,berend:2015MPR-NIPS,muller:2020MPR}):

\begin{equation}\label{eq:accuracy}
\begin{aligned}
\prob(\wmv(\Fv) = \corre) & = \!\! \sum_{\reaset: \reaset \in \dsvw(\pv)} \prob(\reaset)\\
&= \!\! \sum_{\reaset : \prob(\reaset) \geq \prob(-\reaset)} \!\! \left( \prod_{i:\rea_i {=} 1} \!\! p_i \cdot \!\!\! \prod_{i:\rea_i {=} -1} \!\! (1{-}p_i) \right)
\end{aligned} 
\end{equation}

%The decision accuracy of WMV can be determined by $\prob(\dsvw)$.

Equation~\ref{eq:accuracy} indicates that the accuracy of WMV is determined by the probabilities of indicator vectors, which depend on the trustworthiness values of the sources.

In Example~\ref{exp1}, if $\pv= (0.6,0.6,0.9)$, then $\dsvw (\pv) = \{(1,1,1), \\(-1,1,1), (1,-1,1), (-1,-1,1)\}$ and the decision accuracy is 0.9 (i.e., a.l.a source $s_3$ reports correctly).

WMV has been proved to be optimal when \trustworthiness{} $\pv$ and the log weight function are used for decision making and the sources are independent in providing feedback~\cite{nitzan:1982optimal}. 
% Besides, WMV can also be used for multi-option decision-making scenarios~\cite{muller:2020MPR}.

In practice, when trustworthiness is unknown, the weight assigned to each source depends on the \trust{}, that is, $w_i=\log (\nicefrac{\pest_i}{1 - \pest_i})$. Besides, the decision maker computes the probabilities of indicator vectors with trust values, which we use the subscript $\prob_{\pestv}$ to distinguish from their actual probabilities: $\prob_{\pestv}(\reasetv = \reaset) = \prod_{i:\rea_i {=} 1} \pest_i \cdot \prod_{i:\rea_i {=} -1} (1{-}\pest_i)$. 
With $\prob(\reaset)$ replaced by $\prob_{\pestv}(\reaset)$, the decisions would always be correct for those indicator vectors which the decision maker thinks are more probable than their opposite, namely $\dsvw(\pestv)= \{\reaset | \prob_{\pestv}(\reaset) \geq \prob_{\pestv}(-\reaset)\}$.
% Depending on the difference between $\pv,\pestv$,
As a result, $\dsvw(\pestv)$ and $\dsvw(\pv)$ may be different.
In Example~\ref{exp1}, if $\pv = (0.6,0.6,0.9)$ and $\pestv = (0.8,0.6,0.8)$, then $(-1,-1,1) \in \dsvw(\pv)$ while its opposite indicator vector $(1,1,-1) \in \dsvw(\pestv)$. 
This may result in different decision accuracy.
We introduce $\dscpep$ to distinguish: 

\begin{equation}\label{eq:accuracy-pest}
\begin{aligned}
\prob(\wmv(\Fv) = \corre) &\triangleq \dscpep = \sum_{\reaset \in \dsvw(\pestv)} \prob(\reaset)
%&= \sum_{\reaset \in \dsv(\pestv)} \left( \prod_{i:\rea_i {=} 1} p_i \cdot \prod_{i:\rea_i {=} -1} (1{-}p_i) \right)
\end{aligned}
\end{equation}
The first parameter of the function $\omega()$ represents the value used for decision making, and the second parameter represents the value used to compute the probability of deciding correctly.
For $\dscpep$, decisions are made using trust values $\pestv$, while the decision accuracy that the decision maker actually obtains still depends on source trustworthiness, which challenges the optimality of WMV.

%Generally, in Equation~\ref{eq:accuracy-pest}, the first parameter of the function $\omega()$ represents the value used for decision-making, while the second represents the value for computing the probability of deciding correctly.
Generally, both the parameters of $\omega()$ can be either trust or trustworthiness, and we assume that the parameter (either trust or trustworthiness) used for decision-making is at least 0.5. 
Trust values are, by definition, known to the decision-maker. 
Therefore, it’s reasonable to apply the assumption for trust, meaning ignoring sources with trust below 0.5.
For trustworthiness, we assume it is at least 0.5 only when it is used to decide (e.g., in Section~\ref{sub:DSA}), and otherwise, its value ranges from $(0,1)$ (e.g., in Section~\ref{sub:Trustworth SA}, ~\ref{sub:Trust SA} and  ~\ref{sec:stability}).

% We also assume the parameter (either trust or trustworthiness) used for decision-making is at least $0.5$, and the parameter for computing is allowed to take any value.

% Since the parameters can be either trust or trustworthiness, resulting in different meanings for decision accuracy.
Depending on what we equip the parameters with, trust or trustworthiness, we will obtain different meanings for decision accuracy as follows.
The quantity $\dscpp$ denotes the “ideal” decision accuracy, where the decision maker knows and uses the trustworthiness values to decide and compute.
The quantity $\dscpep$ denotes the "practical" decision accuracy, where the decision maker decides with the trust values $\pestv$, but the accuracy he actually achieves depends on \trustworthiness{}.
The quantity $\dscpepe$ denotes the ``perceived'' decision accuracy that the decision maker thinks he can obtain (decides and computes accuracy with trust), while the actual accuracy may not equal $\dscpepe$.
%$\dsc(\pestv,\pestv)$ denotes the \accuracy~ that a decision maker believes he achieve via $\wmv$.

\section{Parameter Sensitivity}\label{sec:sensitivity}
In this section, we analyze how changes in the values of trust and trustworthiness influence the decision accuracy or the \correctness{} of WMV.
There are several ways:
1) how the decision accuracy changes when the trustworthiness and trust change simultaneously; 
2) how the decision accuracy changes with trustworthiness when trust remains constant; 
3) how the decision accuracy changes with trust when trustworthiness remains constant. 
If the changes show relatively little effect on the \correctness{}, then we can say that WMV is not very sensitive to the parameters.
Sensitivity relates to stability, the analysis in this section provides several important insights for the analysis in the next section.

We will also take numerical analysis based on the setting in the following running Example~\ref{exp} to further illustrate the theories.

\begin{myexp}\label{exp}
There are four sources $\sources {=} \{s_1,s_2,s_3,s_4\}$ and their trustworthiness values are $\pv=(0.8,0.75,0.7,0.6)$ respectively.
\end{myexp}

\subsection{Direct Sensitivity Analysis}\label{sub:DSA}
Here, we analyze the case where the parameters used for making decisions and that for computing accuracy are equal.
There are two different rationales for doing this, but the mathematics is identical for both.
First, consider that the decision maker is given the actual \trustworthiness{} values to make decisions.
%If the model is sensitive, then minor changes in \trustworthiness{} affect the correctness significantly.
%This would suggest that relying on \trust{} values instead may be problematic, since they typically contain some noise.
Second, consider analyzing the sensitivity of the beliefs of the decision maker.
Assume the decision maker only knows \trust{} values, and uses them to compute their belief about how probable a decision is correct.
%If the model is sensitive, then minor changes in \trust{} affect the perceived correctness greatly.
%This would suggest that the perceived correctness is not a particularly reliable indicator of correctness.

For simplicity, we use \trustworthiness{} everywhere, but the analysis remains unchanged when using \trust{} instead (simply put a hat on all $p$'s and $\pv$'s). Observe that if \trustworthiness{} of only $1$ source varies, then decision accuracy would appear to be a piecewise linear non-decreasing convex function.
In Figure~\ref{Fig:MPR Property One_Dim}, we depict Example~\ref{exp}, with each plot representing a source trustworthiness variable.
\begin{lemma} \label{lemma:WMV}
	Let $f(p_i) = \dscpp$, where $p_j$ is constant for $j \neq i$.
	The function $f(p_i)$ is a piecewise linear non-decreasing convex function.
\end{lemma}
\begin{proof}[Sketch of Proof]
The computation for correctness of a decision can be characterized as $p_i \cdot x + (1-p_i) \cdot y$, where $x - y \geq 0$ and this coefficient increases with $p_i$ increasing.
Each decision based on corresponding $\dsvw(\pv)$ represents a non-decreasing line.
% and the surface is selected.
% \dw{not very clear, what surface?}
\end{proof}

Generally, if a source is more probable to be trustworthy, the decision will be better and improved faster.
Figure~\ref{Fig:MPR Property One_Dim} illustrates Lemma~\ref{lemma:WMV}.
% Refer to~\ref{eq:wmv-accuracy-computation}, t
% The accuracy of WMV is determined by comparing the probabilities of each pair of realisation and their opposite.
The accuracy of WMV is determined by comparing the $\prob(\reaset)$ and $\prob(-\reaset)$ for all the indicator vectors.
The relation between $\prob(\reaset)$ and $\prob(-\reaset)$ either remains or changes, depending on the value of $p_i$ and how much it changes.
Intuitively, this results in the piece-wise characteristics of $f(p_i)$.
Moreover, we can vary \trustworthiness{} values of multiple (or even all) sources. If we vary $k$ \trustworthiness{}, then we get an $k$-dimensional piecewise surface.
In Figure~\ref{Fig:MPR Property Two_Dims}, we depict our running example with $p_1$ and $p_2$ on the two axes, and $p_3$ and $p_4$ remaining constant.
The surface appears like a collection of intersecting planes, but 
in fact, the graph consists of surfaces described by a polynomial, rather than a linear one.
% rather than the flat ones.
% \todo{any reasons?}

%For high-dimensions, there is no precise quantative result but some bounds of the correctness $\dsc(\pv,\pv)$ in \cite{berend:2015MPR-NIPS,manino:2019NaiveByes_Multiclass}. Below we provide an example to illustrate the functional property.

%We show a numerical simulations of Example \ref{example-property of Weighted Majority Voting} to demonstrate how the accuracy changes with different $\pv$. In Fig. \ref\trustworthiness{}, we plot the change of accuracy $\detercorrect$ on one source $p_i$, and each line represents one source as variable: $p_1$ (solid black), $p_2$ (dashed), $p_3$ (dash-dotted), $p_4$ (dotted).
%As proven in Lemma. \ref{lemma-property of Weighted Majority Voting}$(i)$ and this example, the accuracy $\detercorrect$ piecewisely linear increasing and the slope grows with larger $p_i$. In each segment, the decision set $\decisionset$ is fixed which reasons such segment, and the inflection point is where $\decisionset$ changes a little bit since some $\reaset$ is going to flip. In Fig. \ref{Fig:MPR Property Two_Dims}, w.o.g., we consider $p_1,p_2$ to be the variables, as proven in Lemma. \ref{lemma-property of Weighted Majority Voting}$(ii)$ and this example, the accuracy $\detercorrect$ is a piecewise monotonically increasing function for $\pv$.

% IJCAI Fig.1
\begin{figure}[htbp] %don't force positions
	\centering 
	\subfigure[One Source $p_i$]{
		\includegraphics[width=0.48\linewidth]{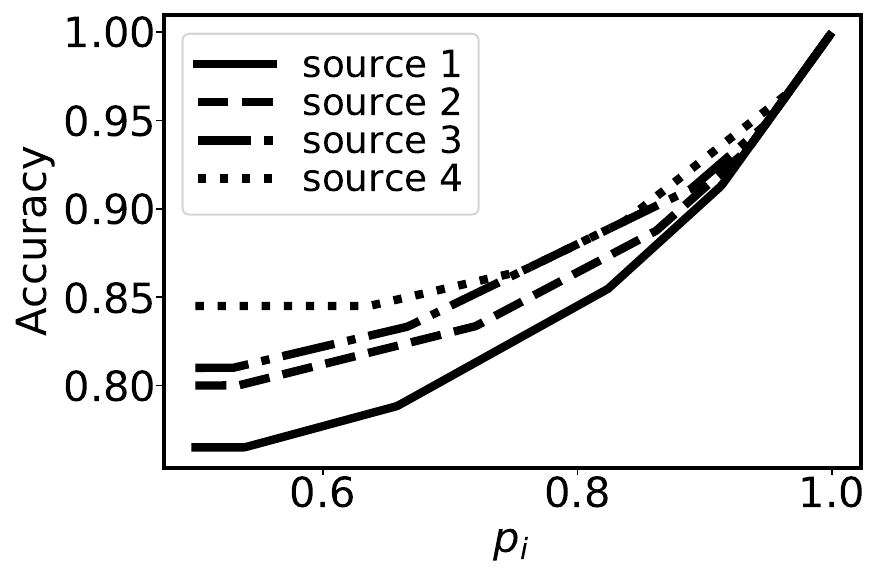}\label{Fig:MPR Property One_Dim}}
	\subfigure[Two sources $p_1$,$p_2$]{
		\includegraphics[width=0.48\linewidth]{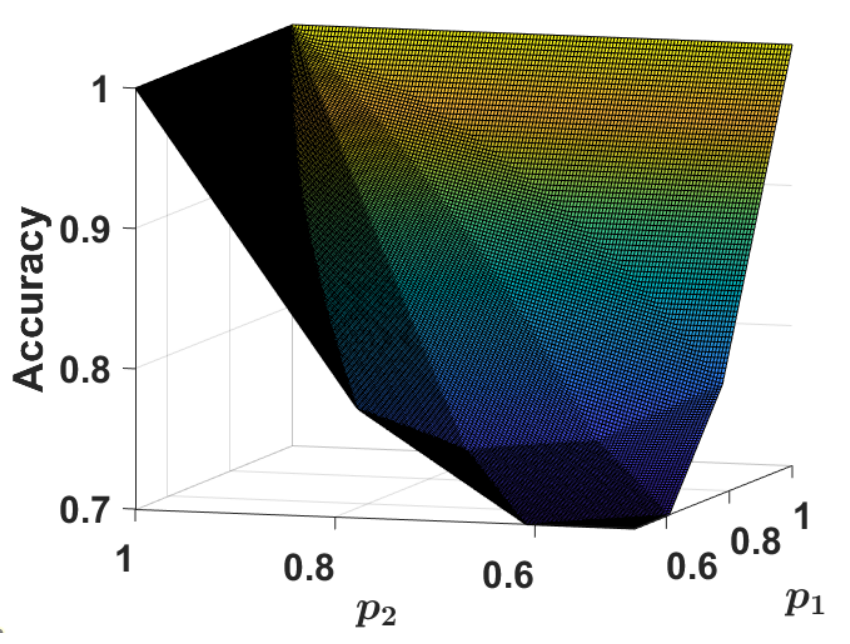}\label{Fig:MPR Property Two_Dims}}
	\caption{Sensitivity of $\wmv$ to $\pv$ when $\pv = \pestv$}
	\label{fig:MPR_Property}
\end{figure}

In Lemma \ref{lemma:WMV}, we assume that trustworthiness of all the sources remains constant and independent. However, sources can collude to influence decisions, or one can update the trust values of multiple sources at a time. For such situations, we assume the trustworthiness values of multiple sources are consistently equal, meaning they are not independent.

\begin{lemma} \label{property of m-identical Weighted Majority Voting}
% 	Let $f(p)=\dscpp$, where $p_i=p$, $i\in\{1,\dots,m\}$, $m\leq n$ and $p_j, j > m$ are constant. 
% 	The function $f(p)$ is a piecewise non-decreasing function, and it is concave (linear or strictly concave) in each segment.
In the special case of the identical sources, let $f(p)=\dscpp$, where $p_1=\dots=p_m=p$, $m\leq n$ and $p_j$ are constant, $j > m$.
The function $f(p)$ is a piecewise non-decreasing function, and it is concave (linear or strictly concave) in each segment.
\end{lemma}

\begin{proof}[Sketch of Proof]
The computation for correctness can be characterized as a summation: $\dscpp = \sum_i g_i(p)$. For each $g_i(p)$, it meets piecewise non-decreasing property, and concave property in each segment. Thus, $\dscpp$ also holds.
\end{proof}

Note that in Lemma~\ref{property of m-identical Weighted Majority Voting}, if $m {=} n$, meaning all the sources are identical, WMV becomes the classical Majority Voting decision rule and $f(p)$ becomes a concave monotonically increasing function ~\cite{Boland:1989_Majority_property}.
%The result of Lemma~\ref{property of m-identical Weighted Majority Voting} also fits well to the case where a group of sources are independent but share the same competence level.

% \begin{example} \label{example-property of m-identical Weighted Majority Voting}
% There are originally $2$ sources with the same $p=0.7$. Then $m$ identical sources join with their $p\geq\frac{1}{2}$ and $n=m+2$.
% \end{example}
% Fig. \ref{fig:Functional property with m sources identical}
Figure \ref{Fig:MPR Property m identical} shows how the decision accuracy $\dscpp$ changes with varying $p$ and $m$ values where there are originally $2$ sources with the same $rest\_{}p=0.7$, then $m$ identical sources join with their $p \geq 0.5$ and $n=m+2$.
Given $m$ value, $f(p)$ increases piecewisely with $p$, and specifically in each segment, it is concavely increasing. 
Given $p$ value, $f(p)$ increases monotonically with $m$.
In Figure \ref{Fig:MPR Property m identical vary rest p}, we fix $n=10,m=6$ and vary the trustworthiness of the $m$ identical sources $p$ and the rest sources $rest\_{}p$. 
% In Figure \ref{Fig:MPR Property m identical vary rest p}, we fix $n=10,m=6$ and vary trustworthiness $p$ and $rest\_{}p$. 
This figure illustrates that even the rest sources are in minority, but the higher $rest\_{}p$ is, the more insensitive the decision to $p$ is. 
Besides, it demonstrates that when the trustworthiness of multiple sources updates in a particular way, the variation characteristic of the decision accuracy may be captured and described.

%IJCAI Fig 2
\begin{figure}[htbp]
	\centering {}
	\subfigure[$rest\_p=0.7$]{
		\includegraphics[width=0.48\linewidth]{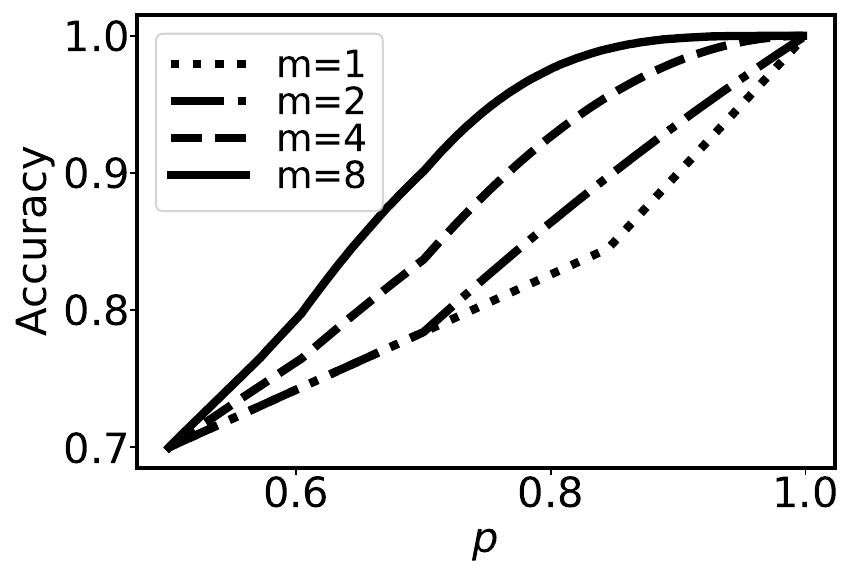}\label{Fig:MPR Property m identical}}
	\subfigure[$n=10$]{
		\includegraphics[width=0.48\linewidth]{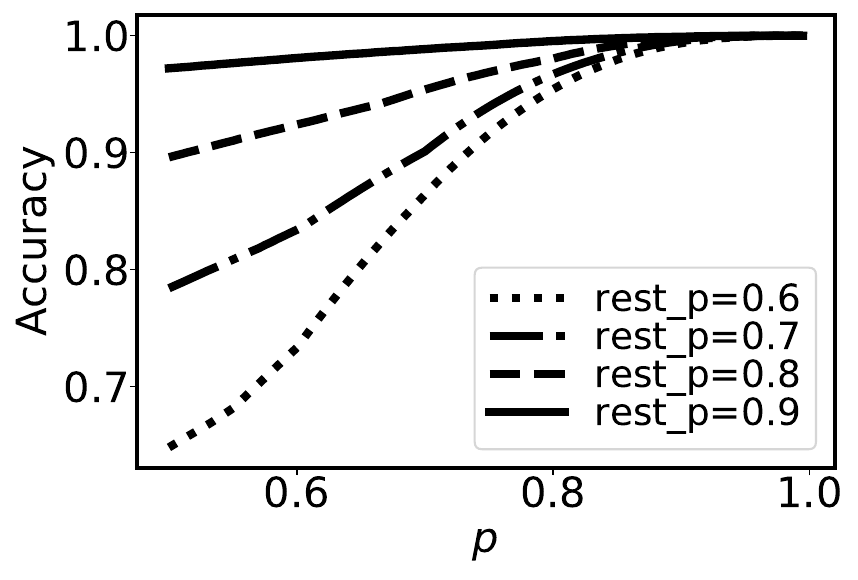}\label{Fig:MPR Property m identical vary rest p}}
	\caption{Accuracy of $\wmv$ with $m$ sources being identical}
	\label{fig:Functional property with m sources identical}
\end{figure}

%NIPS Figure1,2
% \begin{figure}[htbp]
% \centering
% \begin{minipage}[c]{0.48\textwidth}%并排放两张图片，每张占页面的0.5，下同。
% 	\centering{}
% 	\subfigure[One Source $p_i$]{
% 		\includegraphics[width=0.48\linewidth]{Figs/MPR_Property_One_Dim.pdf}\label{Fig:MPR Property One_Dim}}
% 	\subfigure[Two sources $p_1$,$p_2$]{
% 		\includegraphics[width=0.48\linewidth]{Figs/MPR_Property_Two_Dims_Matlab.pdf}\label{Fig:MPR Property Two_Dims}}
% 	\caption{Sensitivity of $\wmv$ to $\pv$ when $\pv = \pestv$}
% 	\label{fig:MPR_Property}
% \end{minipage}
% \begin{minipage}[c]{0.48\textwidth}
% 	\centering{}
% 	\subfigure[rest $p=0.7$]{
% 		\includegraphics[width=0.48\linewidth]{Figs/MPR_Property_m_identical.pdf}\label{Fig:MPR Property m identical}}
% 	\subfigure[$n=10$]{
% 		\includegraphics[width=0.48\linewidth]{Figs/MPR_Property_m_identical_vary_rest_p.pdf}\label{Fig:MPR Property m identical vary rest p}}
% 	\caption{Accuracy of $\wmv$ with $m$ identical sources}
% 	\label{fig:Functional property with m sources identical}
% \end{minipage}
% \end{figure}

\subsection{Trustworthiness Sensitivity Analysis}\label{sub:Trustworth SA}
Next, we analyze the cases where \trustworthiness{} and \trust{} are not identical.
The decisions are made based on the trust values $\pestv$, while the probability of each indicator vector is determined by $\pv$.
The probability of deciding correctly is $\dscpep$.
In this section, \trustworthiness{} varies with \trust{} value fixed.
Recall Equation~\ref{eq:accuracy-pest}, this means that decisions remain unchanged for given feedback (as the set $\dsvw(\pestv)$ remain unchanged), while decision accuracy $\dscpep$ may change with trustworthiness.
% ($\prob(\dsvw(\pestv))$ may change).

If only one parameter $p_i$ varies in $\dscpep$, then the resulting decision accuracy is a non-decreasing function, which follows trivially from the proof of Lemma~\ref{lemma:WMV}. In fact, the line corresponds to one of the line segments from the piece-wise linear graph from the previous section as depicted in
Figure~\ref{Fig:Stability on Inaccurate Estimate, Est_P Fixed One_Dim}. Besides, we can have multiple variables as before.
The surface obtained is non-decreasing and polynomial.
The surface corresponds to one of the fragments from the graph discussed in the previous section.
A 2d example is depicted in Figure~\ref{Fig:Stability on Inaccurate Estimate, Est_P Fixed Two_Dims}. 
The result shows that the decision accuracy has a unique continuous differentiable function, rather than a piecewise function with different functions in different segments.

%IJCAI Figure 3
\begin{figure}[htbp]
	\centering
	\begin{minipage}[c]{0.48\textwidth}
		\centering 
		\subfigure[One Source $p_i$]{
			\includegraphics[width=0.48\linewidth]{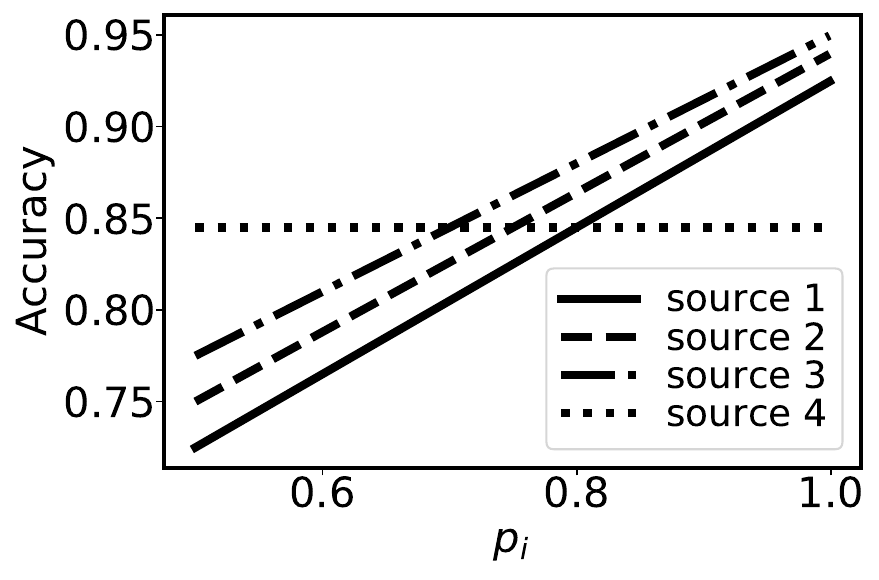}\label{Fig:Stability on Inaccurate Estimate, Est_P Fixed One_Dim}}
		\subfigure[Two sources $p_1$,$p_2$]{
			\includegraphics[width=0.48\linewidth]{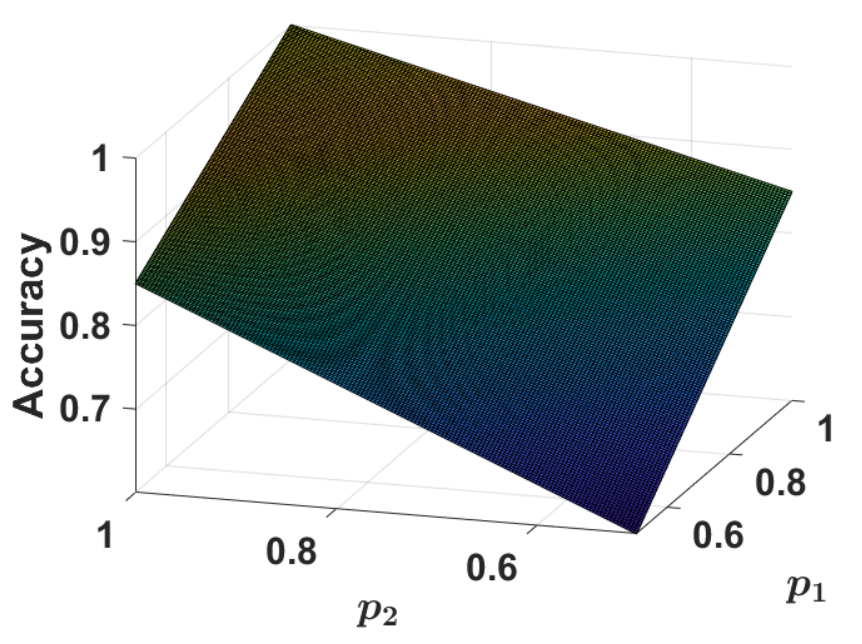}\label{Fig:Stability on Inaccurate Estimate, Est_P Fixed Two_Dims}}
	\end{minipage}
	\caption{Sensitivity of $\pv$ with $\pestv$ fixed}
	\label{fig:Stability on Inaccurate Estimate with Est_P Fixed}
\end{figure}

\subsection{Trust Sensitivity Analysis} \label{sub:Trust SA}
Alternatively, we can take \trust{} to be variable and \trustworthiness{} to be fixed.
Different from before, the actual probability of each indicator vector ($\prob(\reaset)$) now remains unchanged, but the decisions may change (as $\prob_{\pestv}(\reaset)$ and accordingly $\dsvw(\pestv)$ may change).
Then we can analyze what happens if a \trust{} value used for decision making moves away from the actual \trustworthiness{} in either direction.

First, consider the case where we vary the trust value of only one source in $\dscpep$.
We get a uni-modal discontinuous staircase function, which is non-decreasing when $\pest_i < p_i$ and non-increasing when $\pest_i > p_i$.
Figure~\ref{Fig:Stability on Inaccurate Estimate, True_P Fixed One_Dim} depicts Example~\ref{exp} with one variable.
% Second, consider the case where trust values of multiple sources are variable.
Second, when trust values of multiple sources are variable, the resulting surface consists of flat fragments at different heights, with an increasing height with proximity to the point $\pestv = \pv$.
Figure~\ref{Fig:Stability on Inaccurate Estimate, True_P Fixed Two_Dims} depicts Example~\ref{exp} with $\pest_1$ and $\pest_2$ being the variables.
Generally:
\begin{lemma} \label{lemma-Stability on Inaccurate Estimate}
% \sj{When $\pv$ is fixed and $\pestv$ changes resulting in estimate error, $\dscpep$ is a discontinuous staircase function consisting of flat plateaus. Decision accuracy reaches the maximum at the plateau containing the point $\pestv = \pv$. Besides, for the same estimate error $\delta$: $\dsc(\pestv+\delta,\pv)\leq\dsc(\pestv-\delta,\pv)$}
Let $f(\pestv)=\dscpep$, where the \trustworthiness{} $\pv$ is constant. The function $f(\pestv)$ is a discontinuous staircase function consisting of flat plateaus. Decision accuracy reaches the maximum at the plateau containing the point $\pestv = \pv$. 
\end{lemma}
\begin{proof}[Sketch of Proof]
%\sj{A sufficiently large estimation error will lead to wrong judgement on $\prob(\reaset)$ vs. $\prob(-\reaset)$ for some realisations (which is perceived as $\prob_{\pestv}(\reaset)$ vs. $\prob_{\pestv}(-\reaset)$), which is the direct cause of the accuracy degradation with an amount of $\left|\prob(\reaset)-\prob(-\reaset)\right|$.
%Meanwhile, the more difference between \trust{} and \trustworthiness{}, the more the wrong realisations will be involved in $\dsvw(\pestv)$, resulting in worse accuracy.}
The probability that a decision is correct depends on $\pv$, which is constant.
Changing $\pestv$ does not affect the probability that a decision is correct, until it reaches a point where it changes the actual decision away from the optimum.
Then, there is a discontinuous step to a new platform.
\end{proof}
An insight is that the nearby points are more likely to be on the same plateau.
In other words, there is an area of \trust{} values around the \trustworthiness{} values, meaning small estimation deviation may be unlikely to affect the accuracy.
However, it is possible that a certain \trustworthiness{} $\pv$ is exactly at a border (or corner) of a plateau, meaning that even a tiny difference between \trustworthiness{} and \trust{} 
can lead to a staircase difference in correctness.
The positive news is that the plateaus directly bordering the one containing $\pv$ are still more often correct than the ones further away.

Besides, while both underestimation and overestimation cause wrong judgment on $\prob(\reaset)$ vs. $\prob(-\reaset)$, the numerical (Figure \ref{fig:Stability on Inaccurate Estimate with True_P fixed}) results imply that overestimation perhaps results in the worse accuracy degradation compared with underestimation. 
Our intuition is that, if there is a high $p$-valued source, then that source tends to have a lot of sway on the vote, so any inaccuracies will be noticeable, whereas a low $p$-valued source tends to only matter in cases where the vote is tight, and thus any inaccuracies tend to matter less.
From a micro perspective, the underlying reason might be that overestimation of a trustworthiness value makes the difference $\left|\prob(\reaset)-\prob(-\reaset)\right|$ also overestimated, while for underestimation, the difference would be underestimated.
% The underlying reason might be that the overestimation of trustworthiness makes more accuracy degradation $\left|\prob(\reaset)-\prob(-\reaset)\right|$ of some realisations.
% \sj{Underestimation and overestimation cause wrong judgement on $\prob(\reaset)$ vs. $\prob(-\reaset)$ for different realisations.
% The accuracy degradation with amount of  $\left|\prob(\reaset)-\prob(-\reaset)\right|$ of the realisation caused by overestimation may be greater than that caused by underestimation, which is the internal reason of the observation.}
Therefore, when the estimation error is typically inevitable, it is better to underestimate trustworthiness.

%IJCAI Figure4
\begin{figure}[htbp] 
	\centering 
	\begin{minipage}[c]{0.48\textwidth}
		\centering{}
		\subfigure[One Source $\pest_i$]{
			\includegraphics[width=0.48\linewidth]{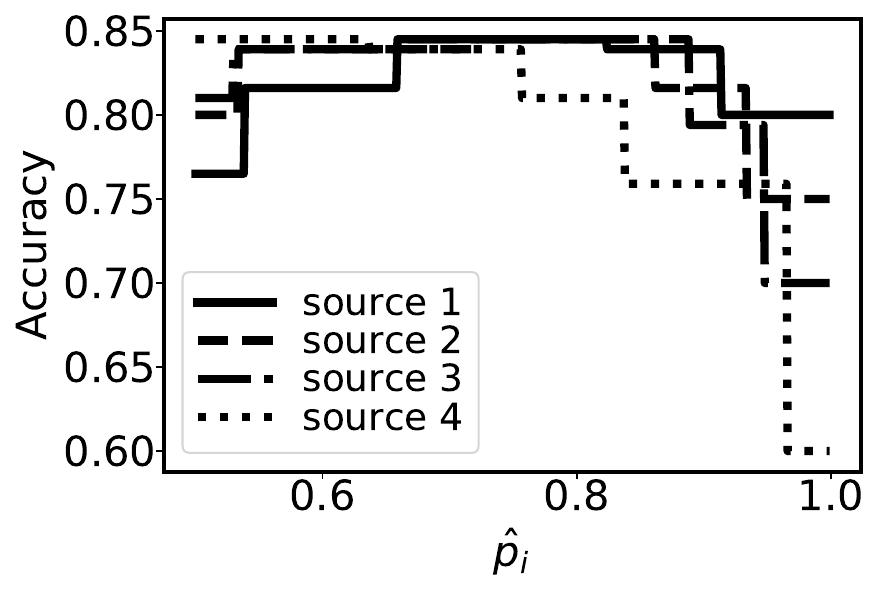}\label{Fig:Stability on Inaccurate Estimate, True_P Fixed One_Dim}}
		\subfigure[Two sources $\pest_1$,$\pest_2$]{
			\includegraphics[width=0.48\linewidth]{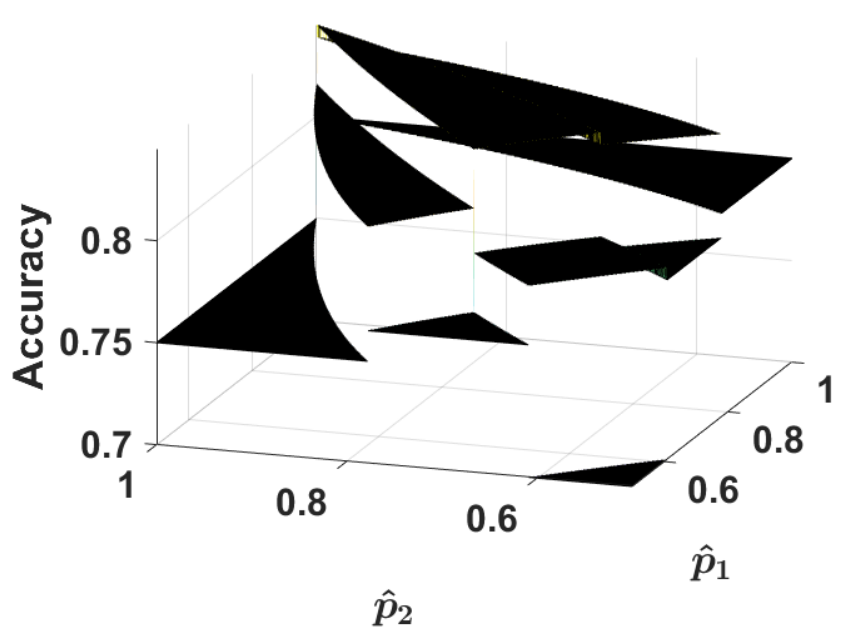}\label{Fig:Stability on Inaccurate Estimate, True_P Fixed Two_Dims}}
	\end{minipage}
	\caption{Sensitivity of $\pestv$ with $\pv$ fixed}
	\label{fig:Stability on Inaccurate Estimate with True_P fixed}
\end{figure}
%NIPS Figure3,4
% \begin{figure}[htbp]
% \centering
% \begin{minipage}[c]{0.48\textwidth}%并排放两张图片，每张占页面的0.5，下同。
% 	\centering{}
% 	\subfigure[One Source $p_i$]{
% 		\includegraphics[width=0.48\linewidth]{Figs/MPR_EstFix_TrueChange_One_Dim.pdf}\label{Fig:Stability on Inaccurate Estimate, Est_P Fixed One_Dim}}
% 	\subfigure[Two sources $p_1$,$p_2$]{
% 		\includegraphics[width=0.48\linewidth]{Figs/MPR_EstFix_TrueChange_Two_Dims_Matlab.pdf}\label{Fig:Stability on Inaccurate Estimate, Est_P Fixed Two_Dims}}
% 	\caption{Sensitivity of $\pv$ with $\pestv$ fixed}
% 	\label{fig:Stability on Inaccurate Estimate with Est_P Fixed}
% \end{minipage}
% \begin{minipage}[c]{0.48\textwidth}
% 	\centering{}
% 	\subfigure[One Source $p_i$]{
% 		\includegraphics[width=0.48\linewidth]{Figs/MPR_TrueFix_EstChange_One_Dim.pdf}\label{Fig:Stability on Inaccurate Estimate, True_P Fixed One_Dim}}
% 	\subfigure[Two sources $p_1$,$p_2$]{
% 		\includegraphics[width=0.48\linewidth]{Figs/MPR_TrueFix_EstChange_Two_Dims_Matlab1.pdf}\label{Fig:Stability on Inaccurate Estimate, True_P Fixed Two_Dims}}
% 	\caption{Sensitivity of $\pestv$ with $\pv$ fixed}
% 	\label{fig:Stability on Inaccurate Estimate with True_P fixed}
% \end{minipage}
% \end{figure}

\section{Stability}\label{sec:stability}
% \tm{
% Here, rather than looking at changing variables, we replace variables by distributions.\\
% We need a good discussion of unbiasedness. It means that given a trust value, the trustworthiness has an expected value equal to it. (But not vice versa)\\
% Stability of \correctness{} means, given that our trust is unbiased, what is the probability we are correct -- knowing that trustworthiness may differ from trust. That means that the trust is constant, but the trustworthiness is chosen from a distribution.\\
% For the Stability of Optimality, there are two equivalent ways of saying the same thing: we can let the trust and trustworthiness both be drawn from distributions with the same expectation, or we can let trust be drawn from a distribution with expectation equal to trustworthiness. This is equivalent.\\
% Note that we missed a third (trivial?) version, where trust and trustworthiness are drawn as equal values. In other words, just applying Jensen's inequality to the thing. We may add this later, if we have time. It should be easy, but also not very important.
% }

The results of the Parameter Sensitivity Section are unsurprising.
Increasing trustworthiness typically increases \correctness{}, and cannot decrease \correctness{}.
Hence, if a source is believed to decide correctly with probability $\pest$ , while its \trustworthiness{} $p < \pest$, then the actual \correctness{} achieved is lower than what the decision maker believes: $\dscpep {<} \dscpepe$.
Vice versa when $p > \pest$.
Our suspicion is that these two effects cancel each other out, if the algorithm that establishes \trust{} is not biased towards overly trusting or being suspicious on average.
We call this property \emph{Stability of Correctness}, and prove it absolutely holds for WMV.
% in Section~\ref{sub:SoC}.

% \dw{The whole section with trust fixed but trustworthiness as a distribution, seems the procedure has nothing to do with ``the reason behind bad or good estimation''. In our study, its actually because of the trustworthiness distribution.}
A better procedure to obtain \trust{} returns values closer to the \trustworthiness{} values, with little variance, meaning what is believed about the sources is close to the ground truth.
The quality of the procedure does not affect the Stability of Correctness at all when it is unbiased, which may initially seem counter-intuitive. However, another property captures the idea that even when it's unbiased, poor \trust{} values still result in worse performance of WMV, \emph{Stability of Optimality}.
We prove Stability of Optimality does not hold absolutely, but that drop in the performance is bounded.

Beforehand, we need to formally define what we mean by an algorithm or procedure to establish \trust{} values, and by it being unbiased (on average). 

\subsection{Parameter Distributions} \label{sub:para dist}
We introduce random variables for our parameters.
For \trustworthiness{}: $\pdis_i$ is a random variable with outcome $p_i \in [0,1]$, and $\pdisv$ is a joint random variable with outcome $\pv = (p_1, \dots, p_n)$.
Similarly, for \trust{}: $\pestdis_i$ is a random variable with outcome $\pest_i \in [0,1]$, and $\pestdisv$ is a joint random variable with outcome $\pestv = (\pest_1, \dots, \pest_n)$.
The uncertainty of source \trustworthiness{} may be due to lack of behavior consistency, or experience, so the sources can not provide stable-quality feedback. On the other hand, inadequate interaction with sources or inaccurate modeling by decision maker may incur uncertain \trust{} estimation.

Weighted Majority Voting requires a weight for each source which is determined by $\pest_i$ (the outcome of $\pestdis_i$).
Practical usage of WMV, therefore, must have some algorithms to arrive at values for $\pestv$.
Depending on the quality of the algorithm, there is a degree of correlation between trust and trustworthiness: $\pestdisv$ and $\pdisv$. We consider the procedure to get the trust values $\pestv$ as \emph{unbiased} when the expectation of trustworthiness equals the trust value: $\expect(\pdisv)= \pestv$.
Hence, if an unbiased trust value $\pest_i$ is $0.7$, then the trustworthiness $\pdis_i$ can sometimes be greater or smaller than $0.7$. 
Note that this is a reasonable assumption for various machine learning-based procedures or Bayesian learning in particular.
% If the initial trust value reflects the average trustworthiness of the population, then all the posterior trust values continue to be unbiased when applying Bayesian learning.
In reality, we cannot guarantee that any machine learning method is completely free of such bias, but the unbiased case is interesting to study, and we expect any residual bias to be fairly small, if the algorithm is configured using sufficient empirical data.

We extend our definition of $\dsc$ to accept random variables as parameters. In that case, the output of $\dsc$ is a distribution over accuracy. The expectation of such decision accuracy is:
\begin{equation}
    \expect(\dscpedispdis) = \sum_{\pestv,\pv}\prob(\pestdisv = \pestv, \pdisv = \pv) \dscpep
\end{equation}
Besides, there can be an ideal situation where "magically" the decision maker knows the actual trustworthiness variable (i.e., $\pdisv$), and can use it to make decisions. The expected probability of making correct decision is, 
\begin{equation}
    \expect(\dscpdispdis) = \sum_{\pv} \prob(\pdisv = \pv) \dscpp
\end{equation}

%In this section, we assume that the decision maker has unbiased estimate $\pestv$ on $\pestdisv$, such that $\forall \pest_i, \pest_i=\expect{\pestdisv_i}$. 
%Setting the empirical weights according to the "plug-in" naive Bayes rule is proven "Bayesian optimal" when modelling $\pestdisv$ as Beta distribution in \cite{berend:2015MPR-NIPS}. Besides, "plug-in" approach is also optimal in practice for arbitrary distribution. \sj{Not proven here.} Then we analyze the effect of "distribution" on Stability of Correctness in Theorem \ref{Theorem:SoC1}, and compare such optimal rule in practice with another ideal rule in Theorem \ref{Theorem:SoO1}.
%Can we recast this text somwhere in the paper?
\subsection{Stability of Correctness}\label{sub:SoC}
In this section, we do not care about what the distribution of $\pdisv$ actually looks like,
% ; it can be any arbitrary distribution (
as long as $\expect(\pdisv) = \pestv$, meaning the trust values used for decision making are unbiased.
% ).
% An example of the distribution of $\pdisv$ is the Beta distribution, which is the result of naive Bayesian learning of trustworthiness, where the decision maker knows in hindsight how often a source gave helpful versus bad feedback.
% For our work, however, it is sufficient to be any distribution.
The main result is that in this case, the decision accuracy that $\wmv$ is believed to achieve by the decision maker, equals the probability that the decision is actually correct.
This is an important positive result, that supports the idea of using WMV in practice.
Decision makers are not delusional about the correctness of their decisions.
Formally, we define the property of Stability of Correctness (SoC) as:

% Our main result:
\begin{theorem}{Stability of Correctness (SoC)}: \label{Theorem:SoC1}
%Suppose that there are $n$ independent sources, for $\forall \pest_i, \pest_i=\expect{\pestdisv_i}$, then
%$$ \Delta_{SoC} = \vert \biasdecorrect - \expect{\estprobcorrect} \vert = 0$$
For WMV, if $\pestv = \expect(\pdisv)$, then $\expect(\dscpepdis) - \dscpepe = 0$.
\end{theorem}
\begin{proof}[Sketch of Proof]
It follows from the fact in Section \ref{sub:Trustworth SA} that the indicator vector set $\dsvw(\pestv)$ where decisions are supposed to be correct remains unchanged, when $\pestv$ is unchanged. Also consider the fact that each $\pdis_i$ is independently distributed.
% \dw{not clear what it means by realisations set being constant.}
\end{proof}

We show the results of two Monte Carlo simulations with $100,000$  runs over Example \ref{exp} to demonstrate the effect of distribution variance on the expected correctness of an unbiased estimate. 
% the decision accuracy.
In Figure~\ref{Fig:Stability on Correctness, Est_P Fixed}, we depict $\dscpepe$ and $\expect(\dscpepdis)$, where trustworthiness $\pdisv$ is a Beta distribution
with expected value $\pestv$ equal to trust (unbiased) 
and a variance set by the $x$-axis. 
This figure shows the variance of the trustworthiness $\pdisv$ has no effect on the correctness on average, which confirms our theorem.
In contrast, in Figure~\ref{Fig:Stability on Correctness, True_P Fixed}, we depict $\dscpp$ and $\expect(\dscpedisp)$, letting the trust be the quantity being a random variable, distributed around trustworthiness with increasing variance.
Unsurprisingly, this figure shows that a more divergent trust distribution leads to lower average correctness $\expect(\dscpedisp)$ since the trust is more likely to be far away from trustworthiness and results in accuracy degradation.
Furthermore, $\expect(\dscpedisp)$ can never exceed $\dscpp$,  in line with the conclusion of section \ref{sub:Trust SA}.
In the next section, we will study this case further.

%Ijcai fig.5
\begin{figure}[htbp] %fored position
	\centering 
	\subfigure[$\pestv$ fixed, $\pdisv$ probabilistic]{
		\includegraphics[width=0.48\linewidth]{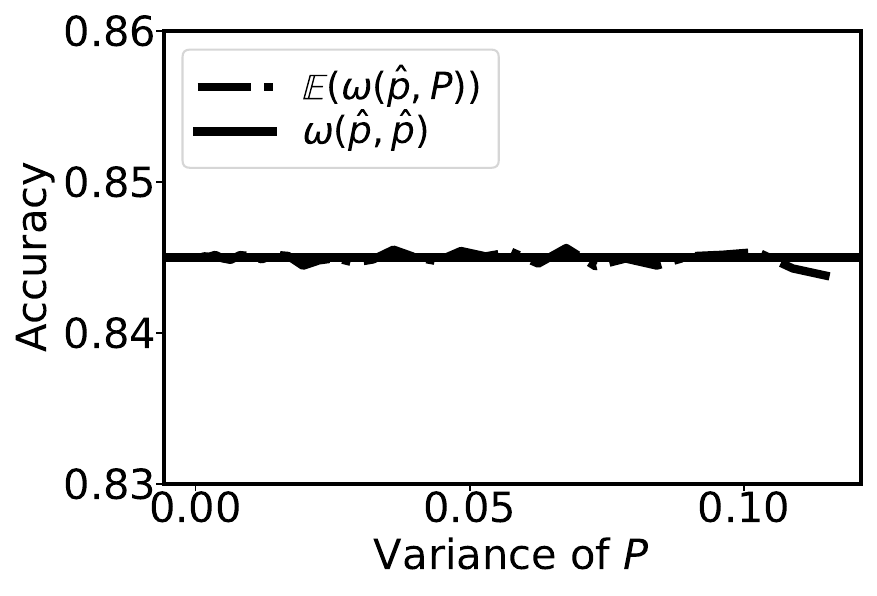}\label{Fig:Stability on Correctness, Est_P Fixed}}
	\subfigure[$\pv$ fixed, $\pestdisv$ probabilistic]{
		\includegraphics[width=0.48\linewidth]{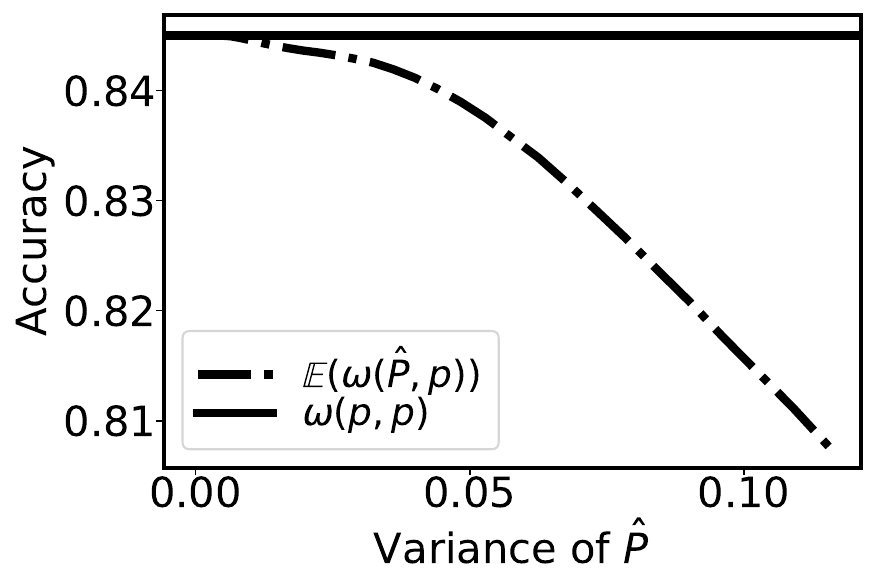}\label{Fig:Stability on Correctness, True_P Fixed}}
	\caption{Effect of variance on Stability of Correctness}
	\label{fig:Effect of Variance on Stability on Correctness with Unbiased estimation}
\end{figure}

\subsection{Stability of Optimality} \label{sub:SoO}
Although for Stability of Correctness the shape (and variance) of the trustworthiness distribution was irrelevant, intuitively a distribution with less variance should still be better for the decision maker.
We introduce another stability property in this section to capture this idea: \emph{Stability of Optimality} (SoO).
% Formally, the measure of stability is the difference in decision accuracy between when the trustworthiness variables are revealed for decision making, and when only trust values are avaliable for decision making.
Formally, it means whether decisions made with the trustworthiness variables revealed are as good as those made only with the trust values available.
We formally capture this gap with the definition below:
\begin{equation}\label{eq:SoO difference}
    SoO(\pdisv) = \expect(\dscpdispdis) - \expect(\dscpepdis)
\end{equation}
In other words, it also measures compared with using trust to decide, how much the decision accuracy can be improved, if trustworthiness values are available.
Note that an equivalent (via Theorem~\ref{Theorem:SoC1}) formulation is: $\expect(\dscpdispdis) - \dscpepe$, when $\pestv = \expect(\pdisv)$.

To analyze Stability of Optimality formally, we introduce some definitions.
Assume the trustworthiness $p_i$ is bounded in some range, e.g. $\forall i, a_i \leq p_i \leq b_i$. 
Denote the value space of $\pv$ as Hypercube $\mathbb{H}$, $\pv \in \mathbb{H}$. 
The set of vertexes of the Hypercube is denoted as Vertex Space $\mathbb{Q}$, where each vertex $\bm{q} \in \mathbb{Q}$ and $q_i$ is either $a_i$ or $b_i$.
Defined in the hypercube, the distribution of $\pdisv$ with expectation $\pestv$ can be arbitrary. 
We name an \emph{extreme distribution} for random variables $\pdisv$ in the vertex space $\mathbb{Q}$ of the hypercube, where $\prob(\pdis_i=a_i) = \frac{b_i-\pest_i}{b_i-a_i}$, $\prob(\pdis_i=b_i) = \frac{\pest_i-a_i}{b_i-a_i}$.

% Observe that a high variance in trustworthiness is \emph{good} for decision accuracy, especially the extreme distribution when trustworthiness is revealed for decision making. 
When trustworthiness is revealed for decision-making, we observe that a high variance in trustworthiness is \emph{good} for accuracy, especially the extreme distribution. 
That is, when a source is more trustworthy than the average, increasing its weight enhances overall decision accuracy. Conversely, when a source is less trustworthy than the average, it can degrade decision quality to some extent, but the impact is mitigated by reducing the weight of this source.
In other words, it's better to have a $50\%$ chance for a source with $P = 0.9$ and $50\%$ for $P = 0.5$, than a source with $p = 0.7$. Formally,

\begin{lemma} \label{lemma-max_distribution}
Take random variables $\pdisv$ defined in a Hypercube with $\expect(\pdisv)=\pestv$. The correctness of $\expect(\dscpdispdis)$ is bounded by the extreme distribution:
\begin{equation}
\expect(\dscpdispdis) \leq \sum_{\bm{q}\in \mathbb{Q}}\left(\dscqq\prod_{i=1}^n\prob(\pdis_i=q_i)\right)
\end{equation}
\end{lemma}
\begin{proof}[Sketch of Proof]
Per Lemma \ref{lemma:WMV}, $\dscpp$ is convex in one dimension, and the extreme distribution maximizes $\expect(\dscpdispdis)$ in that dimension. 
The Lemma follows by independence of the trustworthiness variables.
\end{proof}
Lemma \ref{lemma-max_distribution} demonstrates that the decision accuracy is bounded (not always 100\%), even in the ideal situation where \trustworthiness{} is given, and it is determined by the distribution of \trustworthiness{}.
This is intuitive as more trustworthy sources should lead to better decisions.
%For example, the distribution with more variance would lead better decision accuracy.}
Further, if trustworthiness is a constant rather than a random variable, Lemma \ref{lemma-max_distribution} still holds. That is:
\begin{corollary} \label{lemma-cube_correctness}
For any point $\pv$ in the Hypercube, $\dscpp$ is bounded by a linear combination of the correctness of the vertexes of the hypercube.
\begin{equation}
\dscpp \! \leq \frac{1}{\prod_{i=1}^{n}(b_i-a_i)} \!\! \sum_{\bm{q} \in \mathbb{Q}} \! \dscqq \!\!\! \prod_{i:q_i=a_i} \!\!\! (b_i-p_i) \!\!\! \prod_{i:q_i=b_i} \!\!\! (p_i-a_i)
\end{equation}
\end{corollary}
\begin{proof}
Let $\prob(\pdisv=\pv) = 1$ in Lemma~\ref{lemma-max_distribution}.
\end{proof}
\begin{figure}[htbp] %fored position
	\centering 
	\subfigure[$\pest$ fixed, $P$ probabilistic]{
		\includegraphics[width=0.48\linewidth]{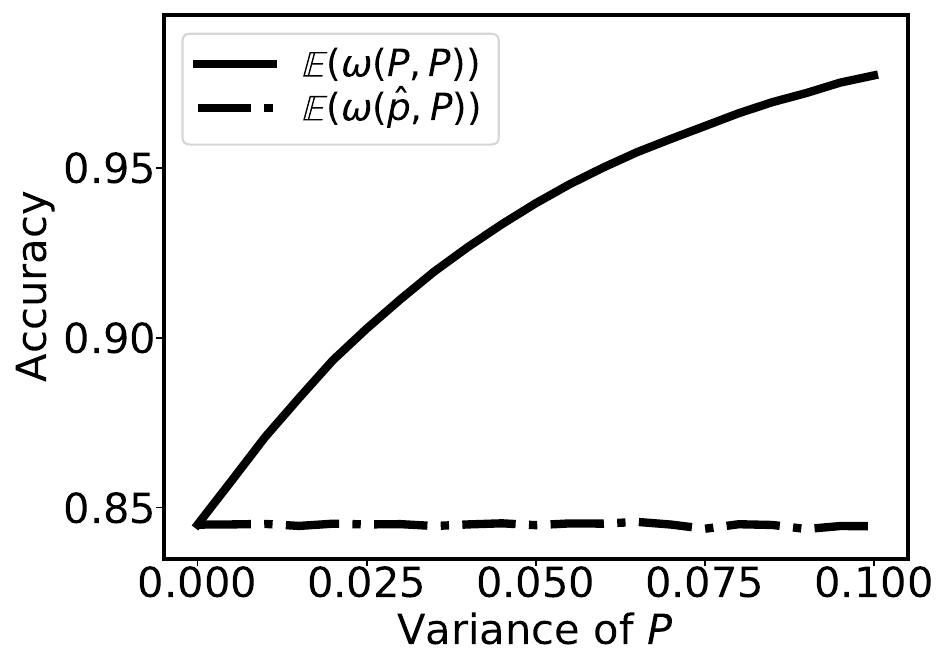}\label{Fig:Stability on Optimality, Est_P Fixed}}
	\subfigure[Example of Beta Distribution]{
		\includegraphics[width=0.48\linewidth]{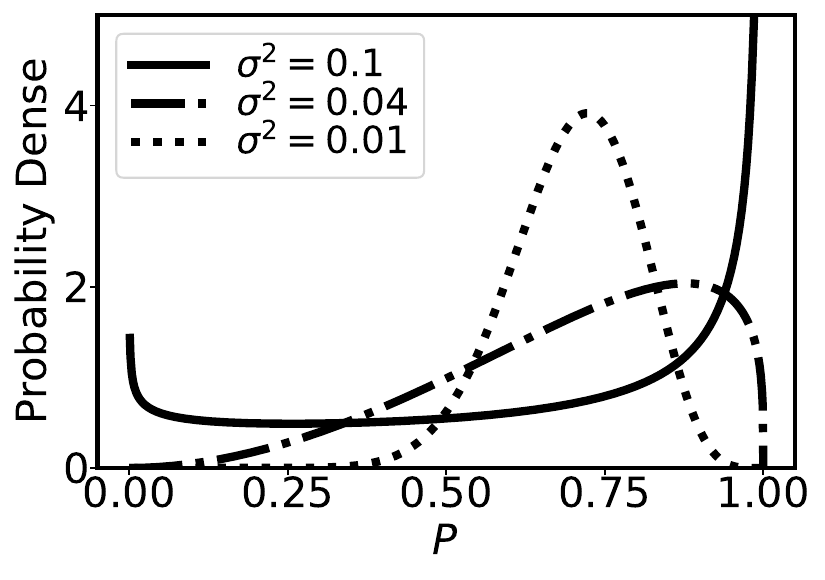}\label{Fig:Example Beta Distribution}}
	\caption{Effect of Variance on Stability of Optimality}
	\label{fig:Effect of Variance on Stability on Optimality with Unbiased estimation}
\end{figure}

%NIPS Figure5,6
% \begin{figure}[htbp]
% \centering
% \begin{minipage}[c]{0.48\textwidth}%并排放两张图片，每张占页面的0.5，下同。
% 	\centering{}
% 	\subfigure[$\pestv$ fixed, $\pdisv$ variables]{
% 		\includegraphics[width=0.48\linewidth]{Figs/MPR_Variance_TrueComptence_Unbiased.pdf}\label{Fig:Stability on Correctness, Est_P Fixed}}
% 	\subfigure[$\pv$ fixed, $\pestdisv$ variables]{
% 		\includegraphics[width=0.48\linewidth]{Figs/MPR_Variance_EstComptence_Unbiased.pdf}\label{Fig:Stability on Correctness, True_P Fixed}}
% 	\caption{Effect of variance on SoC}
% 	\label{fig:Effect of Variance on Stability on Correctness with Unbiased estimation}
% \end{minipage}
% \begin{minipage}[c]{0.48\textwidth}
% 	\centering{}
% 	\subfigure[$\pest$ fixed, $P$ variables]{
% 		\includegraphics[width=0.48\linewidth]{Figs/MPR_SOO_BetaVariance.pdf}\label{Fig:Stability on Optimality, Est_P Fixed}}
% 	\subfigure[Example of Beta($\alpha,\beta$)]{
% 		\includegraphics[width=0.48\linewidth]{Figs/MPR_Variance_TrueComptence_BetaExample.pdf}\label{Fig:Example Beta Distribution}}
% 	\caption{Effect of Variance on SoO}
% 	\label{fig:Effect of Variance on Stability on Optimality with Unbiased estimation}
% \end{minipage}
% \end{figure}

Stability of Optimality does not strictly hold, as the gap (Equation \ref{eq:SoO difference}) is typically non-zero.
We prove upper bounds on the gap, which goes to $0$ as the distribution of trustworthiness becomes tighter.
Let $\delta$ quantify the size of the support of the distribution.

\begin{theorem}{Stability of Optimality:} \label{Theorem:SoO1}
If $\pestv = \expect(\pdisv)$ and all $\pdis_i$ have support $[\pest_i-\delta_i,\pest_i+\delta_i]$, then 
\begin{equation}
% \Delta_{SoO} \leq \left(1-\expect(\dscpepe) \right) \cdot \left(1-\prod_{i=1}^{n}\left(1-\frac{1}{2}\cdot\frac{\delta_i}{1-\pest_i}\right)\right)
SoO(\pdisv) \leq \! \left(1-\dscpepe \right) \! \cdot \!  \left(1-\prod_{i=1}^{n}\left(1-\frac{1}{2}\cdot\frac{\delta_i}{1-\pest_i}\right)\right)
\label{eq:strong}
\end{equation}
A weaker but more intuitive bound is also derived using the Bernoulli Inequality,
\begin{equation}
% \Delta_{SoO} \leq\frac{1-\dscpepe}{2}\sum_{i=1}^{n}\frac{\delta_i}{1-\pest_i}
SoO(\pdisv) \leq\frac{1-\dscpepe}{2}\sum_{i=1}^{n}\frac{\delta_i}{1-\pest_i}
\label{eq:weak}
\end{equation}
\end{theorem}

\begin{proof}[Sketch of Proof]
Via Lemma \ref{lemma-max_distribution}, we know the extreme distribution that maximizes $\expect(\dscpdispdis)$, relying on the correctness of the vertexes. 
Via Corollary \ref{lemma-cube_correctness}, the upper bounds for the correctness of vertexes can be obtained, only relying on $\dscpepe$. 
With some algebra, both bounds~\eqref{eq:strong} and \eqref{eq:weak} can be obtained.
\end{proof}

% Theorem \ref{Theorem:SoO1} analyzes the difference in correctness between using the actual trustworthiness for weights and using the trust for weights. 
% This difference reflects the quantity of the estimation, and indicates to decision makers whether trustworthiness needs to be estimated more accurately to achieve a certain level of decision accuracy.
% \sj{Theorem \ref{Theorem:SoO1} analyzes the difference in correctness between using the actual trustworthiness for weights and using the trust for weights. 
% This difference reflects the uncertainty of \trustworthiness{} indeed will hurt \emph{Stability of Optimality}. Meanwhile, It also describes how well the decision can be improved at most if better estimation is provided.}
 
% While there is a gap between the expected accuracy when unbiased trust is used for decisions and that when trustworthiness is used, 
While there is a gap between making decisions based on unbiased trust and based on trustworthiness,
Theorem~\ref{Theorem:SoO1} proves that this gap is bounded by a relatively small threshold, implying that the unbiased trust would not reduce the decision quality too much. 
The upper bound is influenced by the distribution of trustworthiness, and converges towards zero with that variance reducing. 
% And Theorem \ref{Theorem:SoO1} proves how much they differ \emph{at most}.
% \Theorem \ref{Theorem:SoO1} proves how bad the decision can be \emph{at most} if the inaccurate but unbiased \trust{} is used for decision making.
% That is to say, the decision made based on such \trust{} will not be too bad with accuracy guarantee.
% Besides, the gap $SoO(\pdisv)$ will converge to zero with the size $\delta$ decreasing to zero.
% Somehow, the stability level is determined by the distribution of source \trustworthiness{}, for example, the variance.
% \dw{describe whether and how the parameters on the right side influence this bound.}
% \dw{Does lower right side necessarily means better estimation? or the other direction? we need to be careful about its implication.}
% \sj{The difference has nothing with better or worse estimation, but caused by the uncertainty level of sources \trustworthiness{}. 
% I think we can say "The difference implies that the improvement of decision accuracy perhaps can be greater if better estimation can be obtained, when using sources with more uncertainty."}

To illustrate the effect of distribution variance on $SoO(\pdisv)$, we provide a Monte Carlo simulation with $100,000$ runs over Example \ref{exp}.
In Figure~\ref{Fig:Stability on Optimality, Est_P Fixed}, we measure $\expect(\dscpdispdis)$ and $\expect(\dscpepdis)$, where $\pestv$ is constant, and $\pdisv$ follows Beta distribution with increasing variance.
It presents that the larger the variance is, the larger $SoO(\pdisv)$ is, which validates the result of Lemma \ref{lemma-max_distribution}. 
To put the quantity of the variance in context, we provide examples of trustworthiness being Beta distributions with a certain variance in Figure~\ref{Fig:Example Beta Distribution}.

\begin{figure}[htbp] %fored position
	\centering 
	\begin{minipage}[c]{0.48\textwidth}
		\centering{}
		\subfigure[Effect of $\delta$]{
			\includegraphics[width=0.48\linewidth]{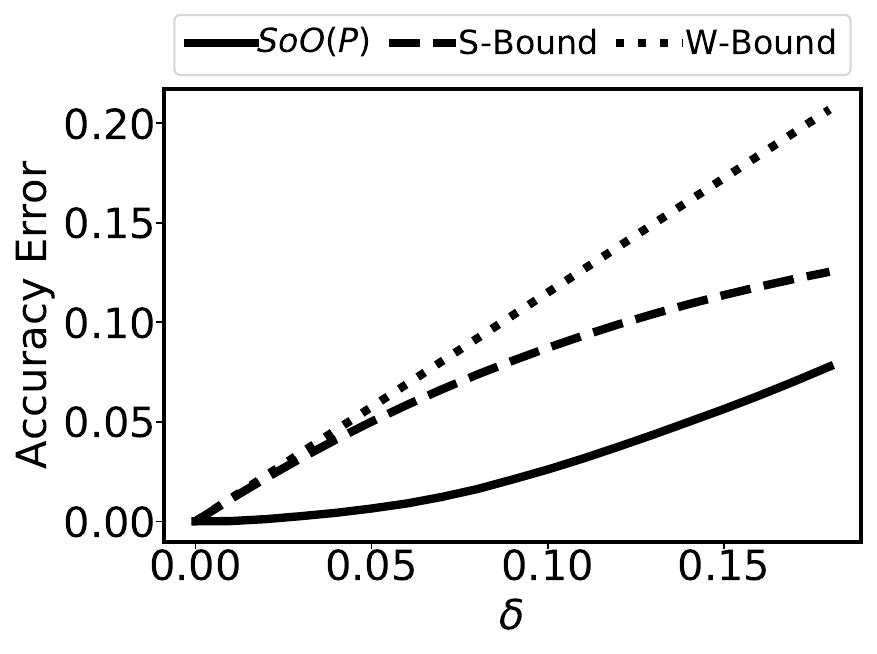}\label{Fig:SoO Parameter Delta}}
		\subfigure[Effect of $n$]{
			\includegraphics[width=0.48\linewidth]{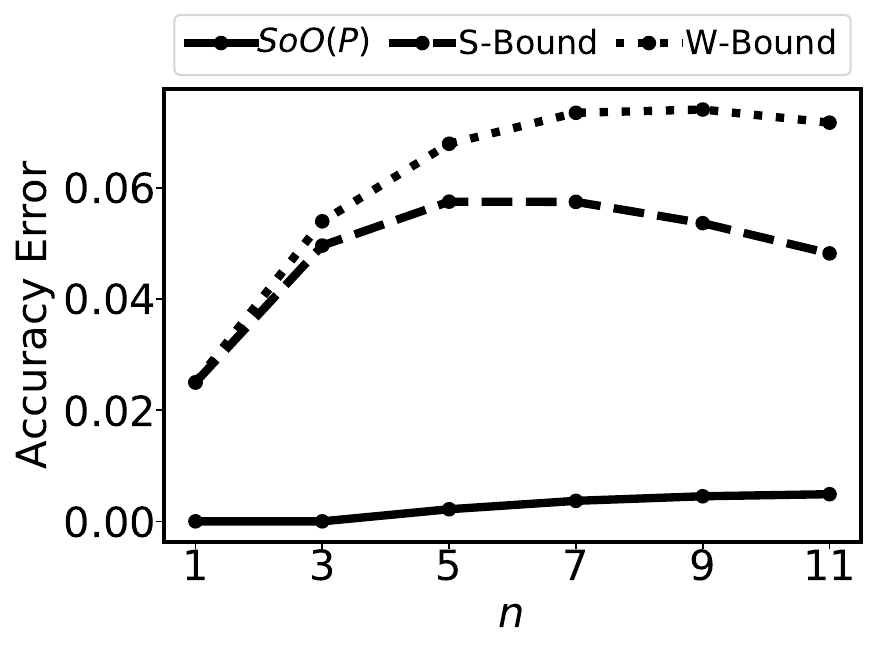}\label{Fig:SoO Parameter n}}
	\end{minipage}\vspace{-2mm}
	\begin{minipage}[c]{0.48\textwidth}
		\centering 
		\subfigure[Effect of single $\pest_1$]{
			\includegraphics[width=0.48\linewidth]{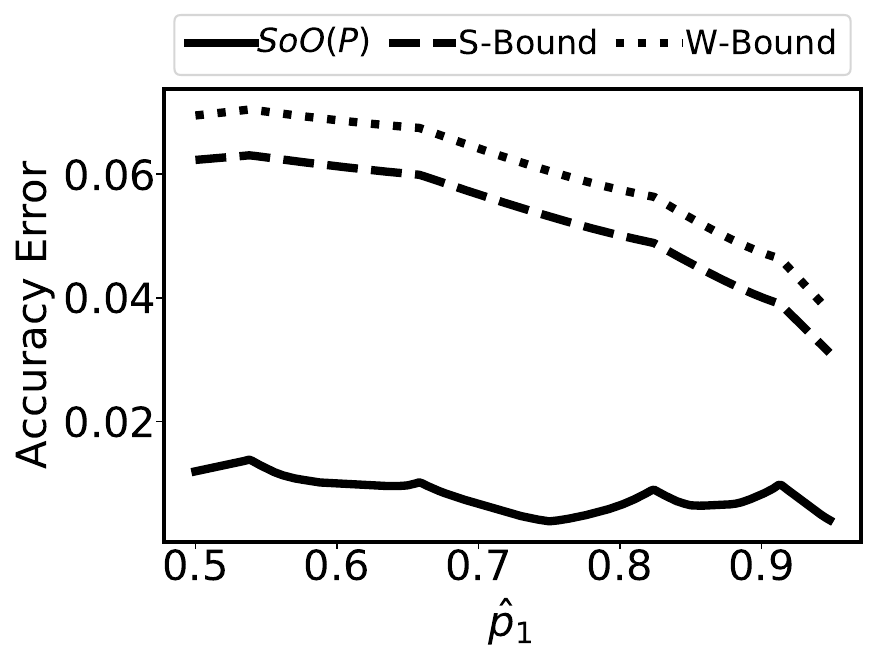}\label{Fig:SoO Parameter P One_Dim}}
		\subfigure[Effect of identical $\pestv$]{
			\includegraphics[width=0.48\linewidth]{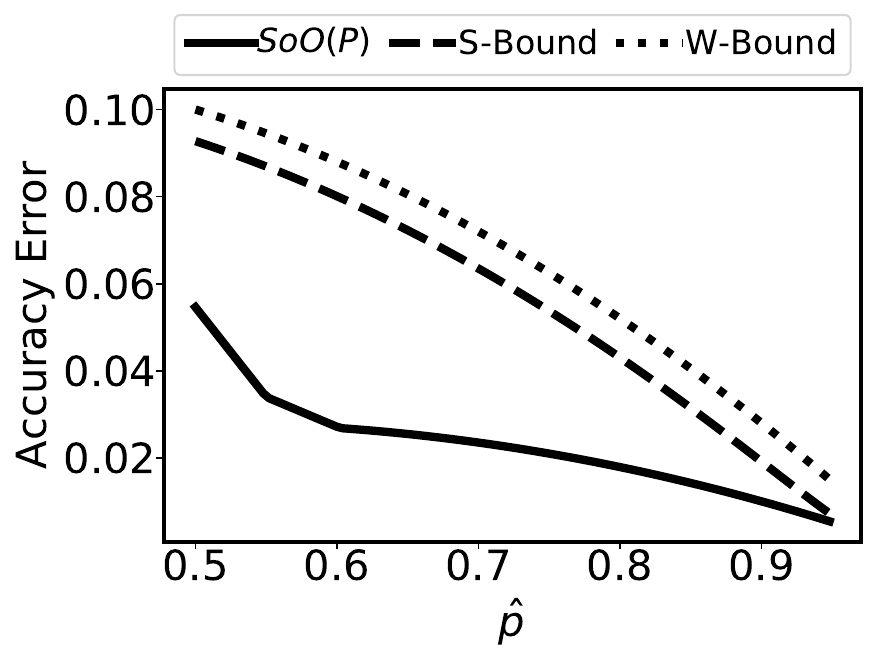}\label{Fig:SoO Parameter P Identical}}
	\end{minipage}
	\caption{Parameter Analysis on Stability of Optimality}
	\label{fig:Parameter Analysis on Stability of Optimality}
\end{figure}

In Figure~\ref{fig:Parameter Analysis on Stability of Optimality}, we provide a parameter analysis with numerical experiments to demonstrate how they influence $SoO(\pdisv)$ and the bounds. 
Example \ref{exp} is used in Figures \ref{Fig:SoO Parameter Delta} and \ref{Fig:SoO Parameter P One_Dim}; Figures \ref{Fig:SoO Parameter n} and \ref{Fig:SoO Parameter P Identical} needed some adaptation, where the sources have identical $\pest=0.7$. 
And $\delta=0.05$ is the default for all the sources. 

Figure~\ref{Fig:SoO Parameter Delta} represents that with $\delta$ decreasing, $SoO(\pdisv)$ and its bounds also decrease to zero. 
There is a linear bound on the effect of $\delta$ (via the weaker bound). 
This implies that the uncertainty level of sources plays a significant role in determining the Stability of Optimality.
In Figure \ref{Fig:SoO Parameter n}, with the number of identical sources $n$ increasing ($\pest=0.7$), $SoO(\pdisv)$ and the bounds change little, which implies source number perhaps influences little on the accuracy gap $SoO(\pdisv)$ (i.e., the gap in accuracy between decision making using unbiased trust and using trustworthiness).
This means that the number of sources may barely influence how valuable it is to know the sources’ combined trustworthiness.

In Figure \ref{Fig:SoO Parameter P One_Dim}, only $\pest_1$ is variable and it shows that $SoO(\pdisv)$ always remains low level and it is a piecewise function with local maximization.
As studied in Section \ref{sub:DSA}, it becomes evident that the local maximization results from the piece-wise nature of the trustworthiness effect on decision accuracy.
In Figure~\ref{Fig:SoO Parameter P Identical}, where all $\pest$ are equal and increase to $1$ simultaneously, $SoO(\pdisv)$ almost decreases to $0$. 
It makes sense because with the trustworthiness increasing, the decision accuracy increases more slowly due to the concavity of WMV with identical sources (See Lemma \ref{property of m-identical Weighted Majority Voting}).
% Based on Figures~\ref{Fig:SoO Parameter P One_Dim}~and~\ref{Fig:SoO Parameter P Identical}, we can also conclude that $SoO(\pdisv)$ is not sensitive to $\pestv$.
% Hence, $SoO(\pdisv)$ is mostly sensitive to increasing $\delta$, but even then sub-linearly.

To conclude, the gap of Stability of Optimality $SoO(\pdisv)$ is somewhat sensitive to parameter $\delta$, which depicts the range and variance of sources \trustworthiness{}, but not sensitive to the number of sources and other parameters.
Overall, the optimality of WMV has a high degree of stability, meaning $SoO(\pdisv)$ tends to be close to $0$.
%\sj{Since the stability mainly depends on $\delta$, as long as $\delta$ is given, the stability can be guaranteed.}

%\section{Discussion}
% \footnote{For more general settings like where there are multiple options and some malicious sources strategically report correctly, then its more advantageous to reason with indicator vector but not feedback~\cite{muller:2020MPR}.}

\section{Conclusion and Future Work}
The common dependence on an estimate or trust of source trustworthiness brings out the need to analyze whether WMV is stable, meaning having tolerant decision inaccuracy with the difference between trust and trustworthiness bounded.

We first analyze how sensitive WMV is to the changes in~\trust{} and~\trustworthiness{}. 
We find that small deviation between \trust{} and \trustworthiness{} does not affect accuracy, and also underestimation usually harms less than overestimation.
We then introduced two statistical properties of WMV, \emph{Stability of Correctness} and \emph{Stability of Optimality}.
Assuming that on average the estimation procedure has no bias towards over or underestimating, we proved that \emph{Stability of Correctness} holds absolutely, regardless of which estimation procedure is used or how well it estimates. 
This guarantees that relying on an unbiased estimate of source trustworthiness is safe, which is also common in practice.
However, the amount of inefficiency introduced by relying on an estimate instead of the trustworthiness itself is limited, as we prove a linear bound on \emph{Stability of Optimality}.
%Overall, our results prove that the decision accuracy obtained actually remains the same as the decision maker would believe when using trust values (SoC) and the quality of the approach only degrades a limited amount when using trust values (SoO).
The proposed formal framework and the two types of stability properties can be generalized to analyze other types of decision mechanisms or scenarios (e.g., where sources are dependent).

%Two statistical properties of WMV, \emph{Stability of Correctness} and \emph{Stability of Optimality}, are then introduced. 
%\sj{The results provide provable and quantifiable bounds to the accuracy degradation of the behaviour of WMV, when the inaccurate but unbiased \trust{} is used for decision making. 
%It implies the decision made based on such \trust{} is not too bad with accuracy guarantee.
%Meanwhile, the distribution of \trustworthiness{} determines the amount of accuracy degradation, with large range or variance leading to large degradation.
%Besides, the decision accuracy is also bounded by the distribution of \trustworthiness{}, even in the best situation -- \trustworthiness{} is revealed.}
%\dw{Not sure which conclusion about SoO is better, Tim help check.}

% For future work, in Stability of Optimality, preciser bounds may be obtained with more detailed information provided, e.g., the variance of \trustworthiness{}. 
% This will further reveal how the distribution of source \trustworthiness{} affects Stability of Optimality under the condition of unbiased estimation.
For future work, beyond the bounded assumption, it's valuable to explore a more precise characterization of the impact of the \trustworthiness{} distribution on SoO in the unbiased setting.
Besides, it is also worth studying the stability of WMV in a more general case, namely when \trust{} is a biased estimate of \trustworthiness{}. 
Some researchers have found that although some sources are assigned weights, they have no influence on the decision result \cite{allouche2021social,bowen2009weighted}.
% \cite{allouche2021social}. 
In other words, we may distribute more estimate error on such sources.

\begin{acks}
This work was supported by National Natural Science Foundation of China (NSFC) under Grant 62106223 and (NSFC) Grant 62293511.
\end{acks}

\bibliographystyle{ACM-Reference-Format}  
\bibliography{stability}

%%%%%%%%%%%%%%%%%%%%%%%%%%%%%%%%%%%%%%%%%%%%%%%%%%%%%%%%%%%%%%%%%%%%%%%%

\end{document}